\documentclass{article}

\usepackage{url}            % simple URL typesetting
\usepackage{booktabs}       % professional-quality tables
\usepackage{amsfonts}       % blackboard math symbols
\usepackage{nicefrac}       % compact symbols for 1/2, etc.
\usepackage{microtype}      % microtypography

\usepackage{amssymb}
\usepackage{amsmath}
\usepackage{amsthm}
\usepackage{bbm}
\usepackage{caption}
\usepackage{color}
\usepackage{enumerate}
\usepackage{enumitem}
\usepackage{float}
\usepackage{graphicx}
\usepackage{hyperref}
\usepackage{stmaryrd}
\usepackage{stmaryrd}
\usepackage{subcaption}
\usepackage{xspace}
\usepackage{float}
\usepackage[numbers]{natbib}
\usepackage{algorithm}
\usepackage[noend]{algorithmic}
\usepackage{wrapfig}
\usepackage{xcolor}

\newtheorem{theorem}{Theorem}[section]
\newtheorem{assumption}{Assumption}[section]

\newtheorem{lemma}[theorem]{Lemma}

\newtheorem{definition}[theorem]{Definition}

\newcommand{\code}[1]{{\texttt {#1}}}

\newcommand{\R}{\mathbb{R}}

\newcommand{\D}{\mathcal{D}}
\renewcommand{\P}{\mathcal{P}}

\newcommand{\M}{\mathcal{M}}
\newcommand{\A}{\mathcal{A}}
\newcommand{\G}{\mathcal{G}}

\newcommand{\p}{\phi}

\newcommand{\traj}{\mathcal{Z}}
\newcommand{\safe}{\text{safe}}
\newcommand{\term}{\text{term}}
\newcommand{\choice}[2]{#1 \; \code{or} \; #2}
\newcommand{\reach}[1]{\code{reach} \; #1}
\newcommand{\avoid}[1]{\code{avoid} \; #1}

\newcommand{\semantics}[1]{{\llbracket #1 \rrbracket}}

\newcommand{\eventually}[1]{\code{achieve} \; #1}
\newcommand{\always}[1]{~ \code{ensuring} \; #1}
\newcommand{\true}{\code{true}}
\newcommand{\false}{\code{false}}

\newcommand{\toolname}{\textsc{Spectrl}\xspace}
\newcommand{\tltl}{\textsc{Tltl}\xspace}
\newcommand{\dirl}{\textsc{DiRL}\xspace}
\newcommand{\dirlnospace}{\textsc{DiRL}}

\usepackage[final]{neurips_2021}

\title{Compositional Reinforcement Learning \\ from Logical Specifications}

\author{%
  Kishor Jothimurugan\\
  University of Pennsylvania\\
  \And
  Suguman Bansal \\
  University of Pennsylvania\\
  \AND
  Osbert Bastani \\
  University of Pennsylvania \\
  \And
  Rajeev Alur \\
  University of Pennsylvania \\
}

\begin{document}

\maketitle

\setlength{\textfloatsep}{12pt}

\begin{abstract}
We study the problem of learning control policies for complex tasks given by logical specifications. Recent approaches automatically generate a reward function from a given specification and use a suitable reinforcement learning algorithm to learn a policy that maximizes the expected reward. These approaches, however, scale poorly to complex tasks that require high-level planning. In this work, we develop a compositional learning approach, called \dirl, that interleaves high-level planning and reinforcement learning. First, \dirl encodes the specification as an abstract graph; intuitively, vertices and edges of the graph correspond to regions of the state space and simpler sub-tasks, respectively. Our approach then incorporates reinforcement learning to learn neural network policies for each edge (sub-task) within a Dijkstra-style planning algorithm to compute a high-level plan in the graph. An evaluation of the proposed approach on a set of challenging control benchmarks with continuous state and action spaces demonstrates that it outperforms state-of-the-art baselines.
\end{abstract}

\section{Introduction}\label{sec:intro}
Reinforcement learning (RL) is a promising approach to automatically learning control policies for continuous control tasks---e.g., for challenging tasks such as walking~\cite{collins2005efficient} and grasping~\cite{andrychowicz2020learning}, control of multi-agent systems~\cite{lowe2017multi,inala2021neurosymbolic}, and control from visual inputs~\cite{levine2016end}. A key challenge facing RL is the difficulty in specifying the goal. Typically, RL algorithms require the user to provide a reward function that encodes the desired task. However, for complex, long-horizon tasks, providing a suitable reward function can be a daunting task, requiring the user to manually compose rewards for individual subtasks. Poor reward functions can make it hard for the RL algorithm to achieve the goal; e.g., it can result in reward hacking~\cite{amodei2016concrete}, where the agent learns to optimize rewards without achieving the goal.

Recent work has proposed a number of high-level languages for specifying RL tasks~\cite{andreas2017modular,li2017reinforcement,spectrl,sun2020program, icarte2018using}. A key feature of these approaches is that they enable the user to specify tasks \emph{compositionally}---i.e., the user can independently specify a set of short-term subgoals, and then ask the robot to perform a complex task that involves achieving some of these subgoals. 
%Then, the system automatically translates this specification to a reward function that can be solved using an existing reinforcement learning algorithm.
% A key benefit of these approaches is that in principle, they expose the compositional structure of the specified task---i.e., the decomposition of the task into a set of subtasks---to the reinforcement learning algorithm. However, existing algorithms have not exploited this structure to improve learning. Instead, they rely on existing end-to-end reinforcement learning algorithms to learn policies that achieve the goals encoded in the given specifications.
Existing approaches for learning from high-level specifications typically generate a reward function, which is then used by an off-the-shelf RL algorithm to learn a policy. Recent works based on Reward Machines \citep{icarte2018using, tor-etal-arxiv20} have proposed RL algorithms that exploit the structure of the specification to improve learning. However, these algorithms are based on model-free RL at both the high- and low-levels instead of model-based RL. Model-free RL has been shown to outperform model-based approaches on low-level control tasks~\cite{chebotar2017combining}; however, at the high-level, it is unable to exploit the large amount of available structure. Thus, these approaches scale poorly to long horizon tasks involving complex decision making.

We propose \dirl, a novel compositional RL algorithm that leverages the structure in the specification to decompose the policy synthesis problem into a high-level planning problem and a set of low-level control problems. Then, it interleaves model-based high-level planning with model-free RL to compute a policy that tries to maximize the probability of satisfying the specification. In more detail, our algorithm begins by converting the user-provided specification into an abstract graph whose edges encode the subtasks, and whose vertices encode regions of the state space where each subtask is considered achieved. Then, it uses a Djikstra-style forward graph search algorithm to compute a sequence of subtasks for achieving the specification, aiming to maximize the success probability. Rather than compute a policy to achieve each subtask beforehand, it constructs them on-the-fly for a subtask as soon as Djikstra's algorithm requires the cost of that subtask.

We empirically evaluate\footnote{Our implementation is available at \href{https://github.com/keyshor/dirl}{https://github.com/keyshor/dirl}.} our approach on a ``rooms environment'' (with continuous state and action spaces), where a 2D agent must navigate a set of rooms to achieve the specification, as well as a challenging ``fetch environment'' where the goal is to use a robot arm to manipulate a block to achieve the specification. We demonstrate that \dirl significantly outperforms state-of-the-art deep RL algorithms for learning policies from specifications, such as \toolname, \tltl, \textsc{Qrm} and \textsc{Hrm}, as well as a state-of-the-art hierarchical RL algorithm, \textsc{R-avi}, that uses state abstractions, as the complexity of the specification increases. In particular, by exploiting the structure of the specification to decouple high-level planning and low-level control, the sample complexity of \dirl scales roughly linearly in the size of the specification, whereas the baselines quickly degrade in performance. Our results demonstrate that \dirl is capable of learning to perform complex tasks in challenging continuous control environments. In summary, our contributions are as follows:
\begin{itemize}
\item A novel compositional algorithm to learn policies for continuous domains from complex high-level specifications that interleaves high-level model-based planning with low-level RL.
\item A theoretical analysis of our algorithm showing that it aims to maximize a lower bound on the satisfaction probability of the specification.
\item An empirical evaluation demonstrating that our algorithm outperforms several state-of-the-art algorithms for learning from high-level specifications.
\end{itemize}

%\label{Sec:Example}
%\begin{figure}
%    \centering
%    \includegraphics[width=0.55\linewidth]{Figures/example.pdf}
%    \caption{Navigation example.}
%    \label{fig:example}
%\end{figure}

\begin{figure*}
\centering
\includegraphics[width=0.2\linewidth]{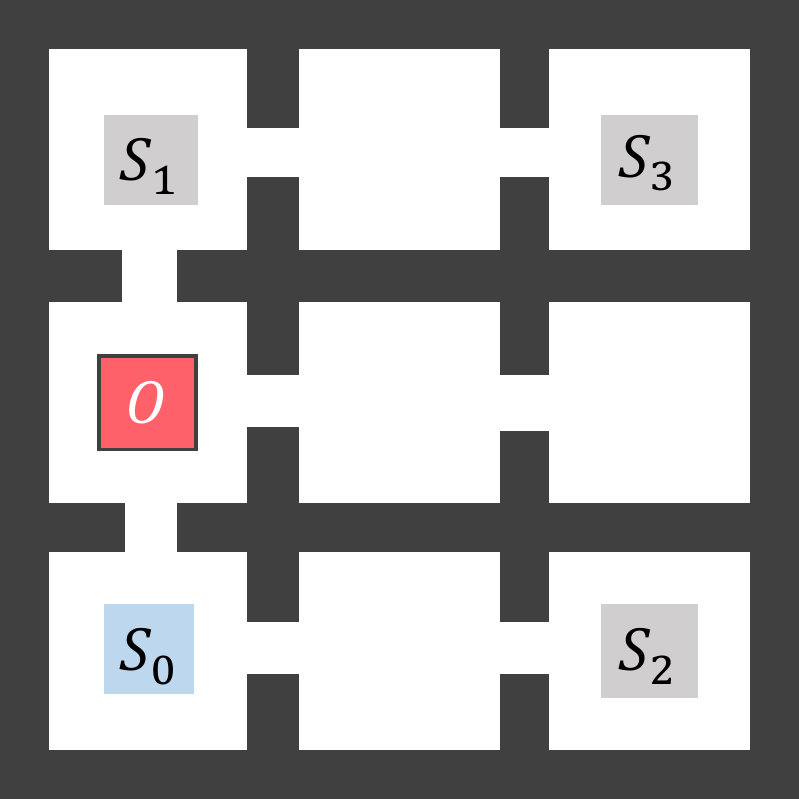}
\qquad
%((\texttt{reach}~u_1\ \texttt{or}\ \texttt{reach}~u_2)\,;\texttt{reach}~u_3)~\texttt{ensuring}\,\not\in O
%\begin{tabular}{c}
%$\small\begin{array}{l}
%\big(\choice{(\eventually{(\reach~r_1)}}{\eventually{(\reach r_2)})}\;;\\
%\hphantom{\big(}\eventually{\reach r_3}\big)~\always{\avoid{O}}
%\end{array}$ \\
\includegraphics[width=0.35\linewidth]{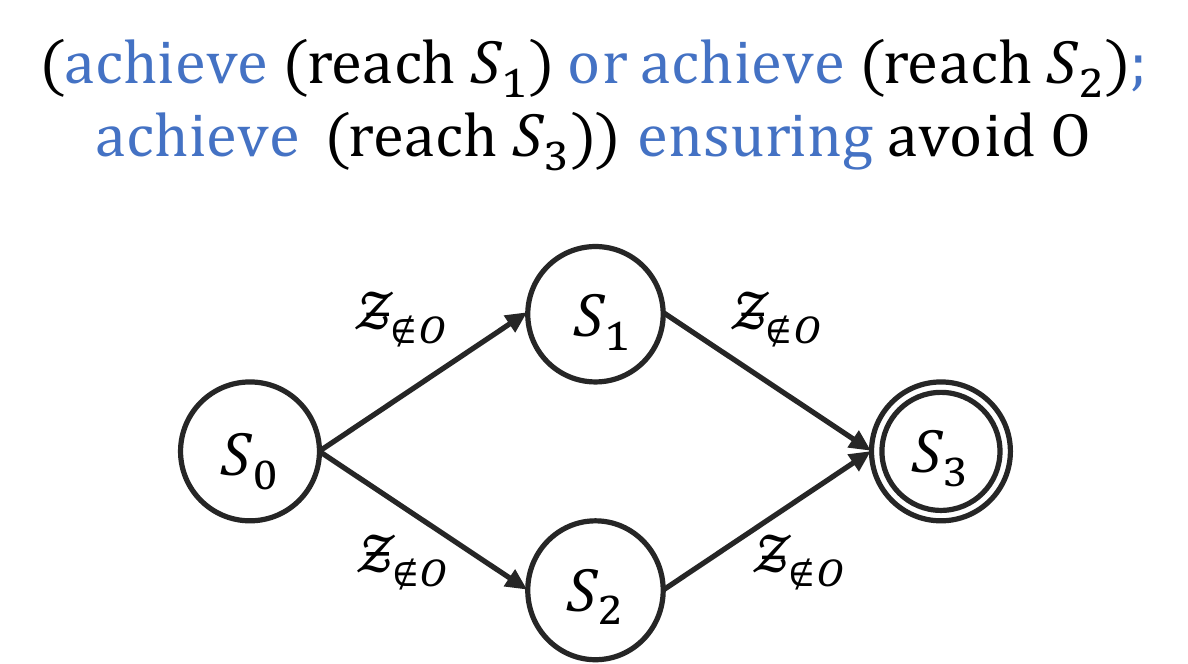}
%\end{tabular}
\qquad
\includegraphics[width=0.31\linewidth]{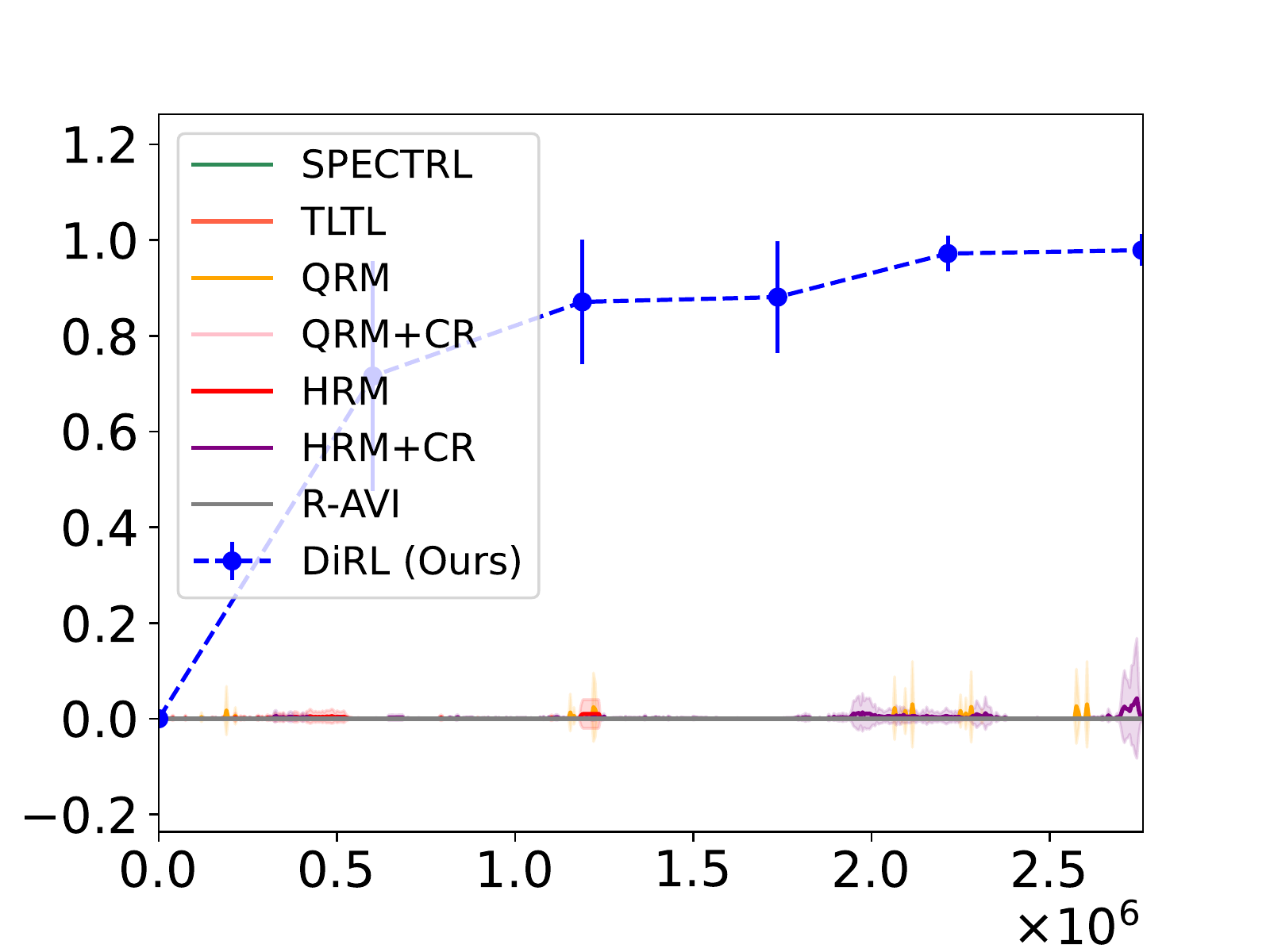}
%\caption{Learning curve for 9-Rooms with specification $\p_{\text{ex}}$. $x$-axis denotes the number of samples (steps) and $y$-axis denotes the probability of success.}
%\label{fig:rooms9greedy}
\caption{Left: The 9-rooms environment, with initial region $S_0$ in the bottom-left, an obstacle $O$ in the middle-left, and three subgoal regions $S_1,S_2,S_3$ in the remaining corners. Middle top: A user-provided specification $\p_{\text{ex}}$. Middle bottom: The abstract graph $\G_{\text{ex}}$ \dirl constructs for $\p_{\text{ex}}$. Right: Learning curves for our approach and some baselines; $x$-axis is number of steps and $y$-axis is probability of achieving $\p_{\text{ex}}$.}
\label{fig:Motivate}
\end{figure*}

\textbf{Motivating example.}
Consider an RL-agent in the environment of interconnected rooms in \autoref{fig:Motivate}.
The agent is initially in the blue box, and their goal is to navigate to either the top-left room $S_1$ or the bottom-right room $S_2$, followed by the top-right room $S_3$, all the while avoiding the red block $O$.
This goal is formally captured by the \toolname specification $\p_{\text{ex}}$ (middle top). This specification is comprised of four simpler RL subtasks---namely, navigating between the corner rooms while avoiding the obstacle. Our approach, \dirl, leverages this structure to improve learning. First, based on the specification alone, it constructs the abstract graph $\G_{\text{ex}}$ 
(see middle bottom) whose vertices represent the initial region and the three subgoal regions, and the edges correspond to subtasks (labeled with a safety constraint that must be satisfied).

%automatically discovering these structural cues from the specification.
However, $\G_{\text{ex}}$ by itself is insufficient to determine the optimal path---e.g., it does not know that there is no path leading directly from $S_2$ to $S_3$, which is a property of the environment. These differences can be represented as (\emph{a priori} unknown) edge costs in $\G_{\text{ex}}$. At a high level, \dirl trains a policy $\pi_e$ for each edge $e$ in $\G_{\text{ex}}$, and sets the cost of $e$ to be $c(e;\pi_e)=-\log P(e;\pi_e)$, where $P(e;\pi_e)$ is the probability that $\pi_e$ succeeds in achieving $e$. For instance, for the edge $S_0\to S_1$, $\pi_e$ is trained to reach $S_1$ from a random state in $S_0$ while avoiding $O$. Then, a na\"{i}ve strategy for identifying the optimal path is to (i) train a policy $\pi_e$ for each edge $e$, (ii) use it to estimate the edge cost $c(e;\pi_e)$, and (iii) run Djikstra's algorithm with these costs.

%bottom-right corner will be simpler (due to the absence of an obstacle), this is not favorable for the specification due to the placement of doors. 

%Our approach uses this structure of the specification to decompose into smaller, easy-to-learn RL sub-tasks as well as to compose the learned policies of the sub-tasks.

One challenge is that $\pi_e$ depends on the initial states used in its training---e.g., training $\pi_e$ for $e=S_1\to S_3$ requires a distribution over $S_1$. Using the wrong distribution can lead to poor performance due to distribution shift; furthermore, training a policy for all edges may unnecessarily waste effort training policies for unimportant edges. To address these challenges, \dirl interweaves training policies with the execution of Djikstra's algorithm, only training $\pi_e$ once Djikstra's algorithm requires the cost of edge $e$. This strategy enables \dirl to scale to complex tasks; in our example, it quickly learns a policy that satisfies the specification with high probability. These design choices are validated empirically---as shown in \autoref{fig:Motivate}, \dirl quickly learns to achieve the specification, whereas it is beyond the reach of existing approaches.
%By gaining in complexity of specifications, efficiency of learning, and accuracy of the generated policy, we improve upon existing approaches that learn from high-level specifications. 

%$(\eventually{(\choice{\reach[(3,4)]}{\reach[(6,0)])}};$ $\reach[(3,7)]) \always{\avoid O}$. 
%This task can be broken down into several simpler subtasks by investigating the structure of the specification. These are: 
%\begin{itemize}
%    \item When the agent is located at $(5,0)$, $\reach[(3,4)]\always \avoid 0$
%    \item When the agent is located  at $(5,0)$, $\reach[(3,7)]\always \avoid 0$
%    \item When agent is located at $(3,4)$, $\reach[(6,0)]\always \avoid 0$
%    \item When agent is located at $(3,7)$, $\reach[(6,0)]\always \avoid 0$.
%\end{itemize}

\textbf{Related Work.}
There has been recent work on using specifications based on temporal logic for specifying RL tasks \cite{aksaray2016q, brafman2018ltlf,de2019foundations, hasanbeig2018logically, littman2017environmentindependent, hasanbeig2019, yuan2019modular, moritz2019, ijcai2019-0557, jiang2020temporallogicbased}. These approaches typically generate a (usually sparse) reward function from a given specification which is then used by an off-the-shelf RL algorithm to learn a policy. In particular, \citet{li2017reinforcement} propose a variant of Linear Temporal Logic (LTL) called \tltl to specify tasks for robots, and then derive shaped (continuous) rewards from specifications in this language. \citet{spectrl} propose a specification language called \toolname that allows users to encode complex tasks involving sequences, disjunctions, and conjunctions of subtasks, as well as specify safety properties; then, given a specification, they construct a finite state machine called a \emph{task monitor} that is used to obtain shaped (continuous) rewards. \citet{icarte2018using} propose an automaton based model called \emph{reward machines} (RM) for high-level task specification and decomposition as well as an RL algorithm (\textsc{Qrm}) that exploits this structure. In a later paper \cite{tor-etal-arxiv20}, they propose variants of \textsc{Qrm} including an hierarchical RL algorithm (\textsc{Hrm}) to learn policies for tasks specified using RMs.  \citet{ijcai2019ltl} show that one can generate RMs from temporal specifications but RMs generated this way lead to sparse rewards. \citet{Kuo2020EncodingFA} and \citet{vaezipoor2021ltl2action} propose frameworks for multitask learning using LTL specifications but such approaches require a lot of samples even for relatively simpler environments and tasks. There has also been recent work on using temporal logic specifications for multi-agent RL \cite{hammond2021, neary2021reward}.

%Reward machines constructed this way have sparse rewards for each subtask, so they are  applicable in environments with continuous state and action spaces.

More broadly, there has been work on using \emph{policy sketches}~\cite{andreas2017modular}, which are sequences of subtasks designed to achieve the goal. They show that such approaches can speed up learning for long-horizon tasks. \citet{sun2020program} show that providing semantics to the subtasks (e.g., encode rewards that describe when the subtask has been achieved) can further speed up learning. There has also been recent interest in combining high-level planning with reinforcement learning~\cite{abel2020value,jothimurugan2021abstract,Eysenbach2019SearchOT}. These approaches all target MDPs with reward functions, whereas we target MDPs with logical task specifications. Furthermore, in our setting, the high-level structure is derived from the given specification, whereas in existing approaches it is manually provided.  \citet{Illanes_Yan_ToroIcarte_McIlraith_2020} propose an RL algorithm for reachability tasks that uses high-level planning to guide low-level RL; however, unlike our approach, they assume that a high-level model is given and high-level planning is not guided by the learned low-level policies. Finally, there has been recent work on applying formal reasoning for extracting interpretable policies \cite{pmlr-v80-verma18a, verma2019imitation, Inala2020SynthesizingPP} as well as for safe reinforcement learning \cite{anderson2020neurosymbolic, junges2016safety}.

%To the best of our knowledge, all of these approaches use an off-the-shelf RL algorithm to learn a policy for the entire task and do not exploit the structure in the specification to perform high-level planning.
%TODO: Options literature? (Planning + RL) literature?

\section{Problem Formulation}\label{sec:prelim}
% We formulate the reinforcement learning (RL) problem for high-level task specifications (instead of low-level reward functions used in standard RL formulations~\cite{henderson2018deep,mnih2015human,sutton2018reinforcement}).

\textbf{MDP.}
We consider a  \emph{Markov decision process (MDP)} $\M = (S, A, P, \eta)$ with continuous states $S \subseteq \R^n$, continuous actions $A \subseteq \R^m$, transitions $P(s,a,s') = p(s'\mid s,a)\in\R_{\geq 0}$ (i.e., the probability density of transitioning from state $s$ to state $s'$ upon taking action $a$), and initial states $\eta: S \rightarrow \R_{\geq 0}$ (i.e., $\eta(s)$ is the probability density of the initial state being $s$). A \emph{trajectory} $\zeta\in\traj$ is either an infinite sequence $\zeta = s_0\xrightarrow{a_0}s_1\xrightarrow{a_1}\cdots$ or a finite sequence $\zeta=s_0\xrightarrow{a_0}\cdots\xrightarrow{a_{t-1}} s_t$ where $s_i \in S$ and $a_i \in A$. A subtrajectory of $\zeta$ is a subsequence $\zeta_{\ell:k} = s_\ell\xrightarrow{a_\ell}\cdots\xrightarrow{a_{k-1}} s_k$.
%, and $P(s_i, a_i, s_{i+1})>0$ for all $i\in\N$.
We let $\traj_f$ denote the set of finite trajectories.
A (deterministic) \emph{policy}  $\pi:\traj_f \to A$ maps a finite trajectory to a fixed action. Given $\pi$, we can sample a trajectory by sampling an initial state $s_0\sim\eta(\cdot)$, and then iteratively taking the action $a_i=\pi(\zeta_{0:i})$ and sampling a next state $s_{i+1}\sim p(\cdot\mid s_i,a_i)$.

\textbf{Specification language.}
We consider the specification language \toolname for specifying reinforcement learning tasks~\cite{spectrl}. A specification $\phi$ in this language is a logical formula over trajectories that indicates whether a given trajectory $\zeta$ successfully accomplishes the desired task. As described below, it can be interpreted as a function $\phi:\traj\to\mathbb{B}$, where $\mathbb{B}=\{\true,\false\}$.
% defined by
% \begin{align*}
% \phi(\zeta)=\mathbb{I}[\zeta~\text{successfully achieves the task}],
% \end{align*}
% where $\mathbb{I}$ is the indicator function.

Formally, a specification is defined over a set of \emph{atomic predicates} ${\P}_0$, where every $p \in {\P}_0$ is associated with a function $\semantics{p}:S\to\mathbb{B}$; we say a state $s$ \emph{satisfies} $p$ (denoted $s\models p$) if and only if $\semantics{p}(s)=\true$. %We consider real-valued outputs (instead of binary outputs) to allow predicates to indicate the degree to which it holds.
For example, given a state $s\in S$, the atomic predicate
$
\semantics{\reach s}(s') ~=~ \big(\|s' - s\| < 1\big)
$
indicates whether the system is in a state close to $s$ with respect to the norm $\|\cdot\|$.
% and given a rectangular region $O\subseteq S$, the atomic predicate
% \begin{align*}
% \semantics{\avoid O}(s) ~=~ (s \not\in O)
% \end{align*}
% indicates if the robot is avoiding $O$. In general, the user can define a new atomic predicate as an arbitrary function $\semantics{p}:S\to\mathbb{B}$.
%
% The set of \emph{predicates} $\mathcal{P}$ consists of conjunctions and disjunctions of atomic predicates. The syntax of a predicate $b\in\mathcal{P}$ is given by the grammar
% % \footnote{Formally, a predicate is a string in the context-free language generated by this context-free grammar.}
% \begin{align*}
% b ~::=~ p \mid (b_1 \wedge b_2) \mid (b_1 \vee b_2),
% \end{align*}
% where $p\in\mathcal{P}_0$.
The set of \emph{predicates} $\mathcal{P}$ consists of conjunctions and disjunctions of atomic predicates. The syntax of a predicate $b\in\mathcal{P}$ is given by the grammar
% \footnote{Formally, a predicate is a string in the context-free language generated by this context-free grammar.}
$
b ~::=~ p \mid (b_1 \wedge b_2) \mid (b_1 \vee b_2),
$
where $p\in\mathcal{P}_0$. Similar to atomic predicates, each predicate $b\in\mathcal{P}$ corresponds to a function $\semantics{b}:S\to\mathbb{B}$ defined naturally over Boolean logic. 
Finally, the syntax of \toolname specifications is given by
\footnote{Here, \code{achieve} and \code{ensuring} correspond to the ``eventually'' and ``always'' operators in temporal logic.}
\begin{align*}
\p ~::=~ \eventually{b} \mid \p_1 \always{b} \mid \p_1; \p_2 \mid \choice{\p_1}{\p_2},
\end{align*}
where $b\in\mathcal{P}$. In this case, each specification $\phi$ corresponds to a function $\semantics{\phi}:\traj\to\mathbb{B}$, and we say $\zeta\in\traj$ satisfies $\phi$ (denoted $\zeta\models\phi$) if and only if $\semantics{\phi}(\zeta)=\true$.
%\footnote{We provided real-valued semantics for predicates, but Boolean semantics for specifications. Prior work provides real-valued (``quantitative'') semantics for specifications as well~\cite{spectrl}, but these are not needed by our algorithm.}
Letting $\zeta$ be a finite trajectory of length $t$, this function is defined by
%, the semantics of $\toolname$ are given by
%Intuitively, the first construct means that the robot should try to reach a state $s$ such that $s\models b$. The second construct says that the robot should try to satisfy $\p_1$ while always staying in states $s$ such that $s\models b$. The third construct says the robot should try to satisfy task $\p_1$ and then task $\p_2$. The fourth construct means that the robot should try to satisfy either task $\p_1$ or task $\p_2$. Formally, the semantics are given by:
\begin{align*}
\zeta&\models\eventually{b} ~&&\text{if}~ \exists\ i \leq t,~s_i\models b \\
\zeta&\models \p \always{b} ~&&\text{if}~ \zeta\models\p ~\text{and}~ \forall\ i\leq t, ~ s_i\models b \\
\zeta&\models\p_1; \p_2 ~&&\text{if}~ \exists\ i < t, ~\zeta_{0:i}\models \p_1 ~\text{and}~ \zeta_{i+1:t}\models\p_2 \\
\zeta&\models\choice{\p_1}{\p_2} ~&&\text{if}~ \zeta\models\p_1 ~\text{or}~ \zeta\models\p_2.
\end{align*}
Intuitively, the first clause means that the trajectory should eventually reach a state that satisfies the predicate $b$. The second clause says that the trajectory should satisfy specification $\p$ while always staying in states that satisfy $b$. The third clause says that the trajectory should sequentially satisfy $\p_1$ followed by $\p_2$. The fourth clause means that the trajectory should satisfy either $\p_1$ or $\p_2$. An infinite trajectory $\zeta$ satisfies $\p$ if there is a $t$ such that the prefix $\zeta_{0:t}$ satisfies $\p$. 

{We assume that we are able to evaluate $\semantics{p}(s)$ for any atomic predicate $p$ and any state $s$. This is a common assumption in the literature on learning from specifications, and is necessary to interpret a given specification $\p$.}
% The size of a specification is defined in terms of the number of temporal operators.
% \renewcommand{\eventually}{\code{achieve} \; }
% \renewcommand{\always}{~ \code{ensuring} \;}
% \renewcommand{\choice}{\; \code{or} \;}
% \renewcommand{\reach}{\code{reach} \; }
% \renewcommand{\avoid}{\code{avoid} \; }

\textbf{Learning from Specifications.}
Given an MDP $\M$ with unknown transitions and a specification $\p$, our goal is to compute a policy $\pi^*:\traj_f\rightarrow\A$  such that
$
\pi^* \in \operatorname*{\arg\max}_{\pi} \Pr_{\zeta\sim\D_{\pi}}[\zeta\models\p],
$
where $\D_{\pi}$ is the distribution over infinite trajectories generated by $\pi$. In other words, we want to learn a policy $\pi^*$ that maximizes the probability that a generated trajectory $\zeta$ satisfies the specification $\p$. 
%Since the transition probabilities are unknown, one can only sample rollouts from $\M$ using a policy starting at any state.
%\autoref{fig:Motivate} shows an example \toolname specification.
%a policy can be learnt by expressing the task as a \toolname specification.
%In this case, the specification can be written as .

% In existing approaches, first the input specification is automatically compiled into a reward function that assigns rewards to rollouts of the MDP. Next, this reward function and the MDP $\M$ are passed as inputs to an off-the-shelf RL algorithm to learn the desired policy $\pi^*$. In contrast, our goal is to leverage the structure of $\phi$ to learn an effective policy $\pi^*$ more quickly.

%
%In particular, the transition probabilities $P(s,a,s')$ are unknown.

We consider the reinforcement learning setting in which we do not know the probabilities $P$ but instead only have access to a simulator of $\M$. Typically, we can only sample trajectories of $\M$ starting at an initial state $s_0\sim\eta$. 
Some parts of our algorithm are based on an assumption that we can sample trajectories starting at any state that has been observed before. For example, if taking action $a_0$ in $s_0$ leads to a state $s_1$, we can store $s_1$ and obtain future samples starting at $s_1$.
%This assumption is justified in finite-state systems since one can store all observed states. 
\begin{assumption}\label{assump:model}
We can sample $p(\cdot\mid s,a)$ for any previously observed state $s$ and any action $a$.
\end{assumption}

%\section{Compositional Reinforcement Learning}
%\label{Sec:LearningAlgo}

\section{Abstract Reachability}
\label{sec:reach}

In this section, we describe how to reduce the RL problem for a given MDP $\M$ and specification $\p$ to a reachability problem on a directed acyclic graph (DAG) $\G_\p$, augmented with information connecting its edges to subtrajectories in $\M$. In Section~\ref{sec:algo}, we describe how to exploit the compositional structure of $\G_\p$ to learn efficiently.

\subsection{Abstract Reachability}

We begin by defining the \emph{abstract reachability} problem, and describe how to reduce the problem of learning from a $\toolname$ specification to abstract reachability. At a high level, abstract reachability is defined as a graph reachability problem over a directed acyclic graph (DAG) whose vertices correspond to \emph{subgoal regions}---a {subgoal region} $X\subseteq S$ is a subset of the state space $S$.
As discussed below, in our reduction, these subgoal regions are derived from the given specification $\phi$. The constructed graph structure also encodes the relationships between subgoal regions.
\begin{definition}
\rm
An {\em abstract graph} $\G = (U,E,u_0,F,\beta,\traj_{\safe})$ is a directed acyclic graph (DAG) with vertices $U$,
(directed) edges $E\subseteq U\times U$, initial vertex $u_0\in U$, final vertices $F\subseteq U$, subgoal region map $\beta:U\to2^S$ such that for each $u\in U$, $\beta(u)$ is a subgoal region,\footnote{We do not require that the subgoal regions partition the state space or that they be non-overlapping.} and \emph{safe trajectories}
$
\traj_\safe = \bigcup_{e \in E}\traj_\safe^e,
$
%\begin{align*}
%\traj_\safe = \bigcup_{e\in E}\traj_\safe^e %\cup \bigcup_{u\in F}\traj_\safe^u
%\end{align*}
where $\traj_\safe^e\subseteq\traj_f$ denotes the safe trajectories for edge $e \in E$.
%where $\traj_\safe^e\subseteq\traj$ denotes the safe trajectories over an edge $e \in E$.
\end{definition}
%Additionally, we associate a subset of safe rollouts with each edge and final vertex.
Intuitively, $(U,E)$ is a standard DAG, and $u_0$ and $F$ define a graph reachability problem for $(U,E)$. Furthermore, $\beta$ and $\traj_{\safe}$ connect $(U,E)$ back to the original MDP $\M$; in particular, for an edge $e=u\to u'$, $\traj_{\safe}^e$ is the set of trajectories in $\M$ that can be used to transition from $\beta(u)$ to $\beta(u')$.
%where $\traj_{\safe}^e$ and $\traj_{\safe}^u$ are sets of safe trajectories for each edge $e\in E$ and  each final abstract state $u\in F$, respectively. 
%\todo{Given an MDP? The rollout is present in an MDP, right? Even the policy $\overline{\pi}$ is defined in the MDP, isn't?}
\begin{definition}
\rm
An infinite trajectory $\zeta=s_0\xrightarrow{a_0}s_1\xrightarrow{a_1}\cdots$ in $\M$ satisfies \emph{abstract reachability} for $\G$ (denoted $\zeta\models \G$) if there is a sequence of indices $0=i_0\leq i_1<\cdots<i_k$ and a path $\rho=u_0\to u_1\to\cdots\to u_k$ in $\G$ such that
\begin{itemize}[topsep=0pt,itemsep=0ex,partopsep=1ex,parsep=1ex]
\item $u_k\in F$,
\item for all $j\in\{0,\ldots,k\}$, we have $s_{i_j}\in \beta(u_j)$, and
\item for all $j < k$, letting $e_j=u_j\to u_{j+1}$, we have $\zeta_{i_j:i_{j+1}}\in\traj_{\safe}^{e_j}$.
\end{itemize}
\end{definition}
The first two conditions state that the trajectory should visit a sequence of subgoal regions corresponding to a path from the initial vertex to some final vertex, and the last condition states that the trajectory should be composed of subtrajectories that are safe according to $\traj_\safe$.
\begin{definition}
\rm
Given MDP $\M$ with unknown transitions and abstract graph $\G$, the \emph{abstract reachability problem} is to compute a policy $\tilde{\pi}:\traj_f\to A$ such that 
$
\tilde{\pi} \in \operatorname*{\arg\max}_{\pi}\Pr_{\zeta\sim\D_{\pi}}[\zeta\models \G].
$
\end{definition}
In other words, the goal is to find a policy for which the probability that a generated trajectory satisfies abstract reachability is maximized.

\subsection{Reduction to Abstract Reachability}

Next, we describe how to reduce the RL problem for a given MDP $\M$ and a specification $\p$ to an abstract reachability problem for $\M$ by constructing an abstract graph $\G_\p$ inductively from $\p$. We give a high-level description here, and provide details in Appendix~\ref{sec:reduction} in the supplement.

\begin{figure}
%\begin{wrapfigure}{r}{0.28\textwidth}
\begin{center}
\includegraphics[width=0.3\textwidth]{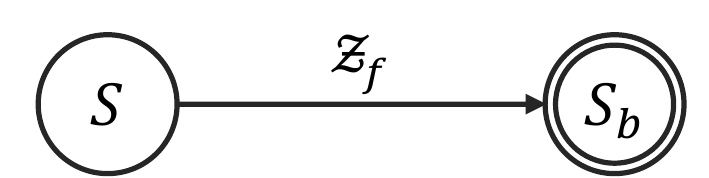}
\end{center}
\caption{Abstract graph for $\eventually{b}$.}
\label{fig:ev_graph}
%\end{wrapfigure}
\end{figure}

First, for each predicate $b$, we define the corresponding subgoal region
$
S_b=\{s\in S\mid s\models b\}
$
denoting the set of states at which $b$ holds.
Next, the abstract graph $\G_\p$ for $\p = \eventually{b}$ is shown in Figure~\ref{fig:ev_graph}. All trajectories in $\traj_f$ are considered safe for the edge $e=u_0\to u_1$ and the only final vertex is $u_1$ with $\beta(u_1) = S_b$. The abstract graph for a specification of the form $\p = \p_1\always{b}$ is obtained by taking the graph $\G_{\p_1}$ and replacing the set of safe trajectories 
$\traj_{\safe}^e$, for each $e\in E$, with the set $\traj_{\safe}^e\cap\traj_{b}$, where
%$\traj_{\safe}^e$ for $e\in E$, with the set $\traj_{\safe}^e\cap\traj_{b}$, where
$
\traj_{b} = \{\zeta\in\traj_f\mid \forall i\;.\; s_i\models b\}
$
is the set of trajectories in which all states satisfy $b$.
For the sequential specification $\p = \p_1;\p_2$, we construct $\G_\p$ by adding edges from every final vertex of $\G_{\p_1}$ to every vertex of $\G_{\p_2}$ that is a neighbor of its initial vertex. Finally, choice $\p=\p_1\;\code{or}\;\p_2$ is handled by merging the initial vertices of the graphs corresponding to the two sub-specifications. 
%Each subgoal region $u$ in $\G_\p$ corresponds to a predicate $b_u\in\P$ of $\p$ such that $u = $.
\autoref{fig:Motivate} shows an example abstract graph. The labels on the vertices are regions in the environment. All trajectories that avoid hitting the obstacle $O$
%(i.e., the set $\traj_{\notin O}$),
are safe for all edges. We have the following key guarantee:
\begin{theorem}\label{thm:reduction}
Given a $\toolname$ specification $\p$, we can construct an abstract graph $\G_\p$ such that, for every infinite trajectory $\zeta \in \traj$, we have $\zeta\models\p$ if and only if $\zeta\models\G_\p$. Furthermore, the number of vertices in $\G_\p$ is $O(|\p|)$ where $|\p|$ is the size of the specification $\p$.
\end{theorem}
%\begin{proof}[Proof Sketch]
%\end{proof}
%\begin{corollary}
%For all policies, the probabilities of learning from spec and its abstract graph is the same. \end{corollary}
We give a proof in Appendix~\ref{sec:reduction}. As a consequence, we can solve the reinforcement learning problem for $\phi$ by solving the abstract reachability problem for $\G_\p$. As described below, we leverage the structure of $\G_\p$ in conjunction with reinforcement learning to do so.

\section{Compositional Reinforcement Learning}\label{sec:algo}
In this section, we propose a compositional approach for learning a policy to solve the abstract reachability problem for MDP $\M$ (with unknown transition probabilities) and abstract graph $\G$.

\subsection{Overview}

At a high level, our algorithm proceeds in three steps:
\begin{itemize}
\item For each edge $e=u\to u'$ in $\G$, use RL to learn a neural network (NN) policy $\pi_e$ to try and transition the system from any state $s\in \beta(u)$ to some state $s'\in \beta(u')$ in a safe way according to $\mathcal{Z}_{\text{safe}}^e$. Importantly, this step requires a distribution $\eta_u$ over initial states $s\in \beta(u)$.
\item Use sampling to estimate the probability $P(e;\pi_e,\eta_u)$ that $\pi_e$ safely transitions from $\beta(u)$ to $\beta(u')$.
\item Use Djikstra's algorithm in conjunction with the edge costs $c(e)=-\log(P(e;\pi_e,\eta_u))$ to compute a path $\rho^*=u_0\to u_1\to\cdots\to u_k$ in $\G$ that minimizes
$
c(\rho)=-\sum_{j=0}^{k-1}\log(P(e_j;\pi_j,\eta_j)),
$
where $e_j=u_j\to u_{j+1}$, $\pi_j=\pi_{e_j}$, and $\eta_j=\eta_{u_j}$. %That is, computes the minimum cost path. 
\end{itemize}
Then, we could choose $\pi$ to be the sequence of policies $\pi_1,...,\pi_{k-1}$---i.e., execute each policy $\pi_j$ until it reaches $\beta(u_{j+1})$, and then switch to $\pi_{j+1}$. 

There are two challenges that need to be addressed in realizing this approach effectively. First, it is unclear what distribution to use as the initial state distribution $\eta_u$ to train $\pi_e$.
%how to sample start states for training these policies since the distribution over states in a source subgoal region depends on the learned policies as well as the high level plan. 
Second, it might be unnecessary to learn all the policies since a subset of the edges might be sufficient for the reachability task. Our algorithm (Algorithm~\ref{alg:dijkstra}) addresses these issues by lazily training $\pi_e$---i.e., only training $\pi_e$ when the edge cost $c(e)$ is needed by Djikstra's algorithm.
%by learning a policy corresponding to an edge only when the cost of the edge is queried by the planning algorithm.

%\textbf{State distributions.}
%%
%As described above, the main issue with this approach is what to choose for $\eta_u$. To address this issue, our algorithm keeps track of the path policy $\tilde{\pi}_u$ for transitioning the system from $u_0$ to $u$ with the highest success probability $P(\tilde{\pi})$. Then, it chooses $\eta_u=\eta_{\tilde{\pi}_u}$ to be the state distribution induced by using $\tilde{\pi}_u$.

\begin{algorithm}[tb]
\caption{Compositional reinforcement learning algorithm for solving abstract reachability.}
\label{alg:dijkstra}
\begin{algorithmic}
\FUNCTION{\dirlnospace($\M$, $\G$)}
%\STATE Initialize policies $\pi_e$ for $e\in E$ and $\pi_u$ for $u\in F$
%\STATE Initialize state distributions $\eta_u$
\STATE Initialize processed vertices $U_p\gets\varnothing$
\STATE Initialize $\Gamma_{u_0}\gets\{u_0\}$, and $\Gamma_u\gets\varnothing$ for $u\neq u_0$
\STATE Initialize edge policies $\Pi\gets\varnothing$
%\STATE Initialize policies $\pi_e$ for $e\in E$
%\STATE Set $c_u \gets \infty$ for every $u\in U\setminus\{u_0\}$
%\STATE Set $c_{u_0}\gets 0$
\WHILE{\textbf{true}}
\STATE $u \gets \textsc{NearestVertex}(U\setminus U_p, \Gamma, \Pi)$
\STATE $\rho_u\gets\textsc{ShortestPath}(\Gamma_u)$
\STATE $\eta_u\gets\textsc{ReachDistribution}(\rho_u,\Pi)$
\STATE \textbf{if} {$u\in F$} \textbf{then return} $\textsc{PathPolicy}(\rho_u,\Pi)$
\FOR{$e = u \to u' \in \text{Outgoing}(u)$}
\STATE $\pi_e \gets \textsc{LearnPolicy}(e, \eta_u)$
\STATE Add $\rho_u\circ e$\ to\  $\Gamma_{u'}$ and $\pi_e$ to $\Pi$
%\STATE Let $c'_{u_2} = c_{u_1} - \log(P(u_1\to u_2;\eta_{u_1}))$
%\IF{$c'_{u_2} < c_{u_2}$}
%\STATE $c_{u_2}\gets c'_{u_2}$
%\STATE $\eta_{u_2}\gets \eta[u_1\to u_2,\eta_{u_1}]$
%\ENDIF
\ENDFOR
\STATE Add $u$ to $U_p$
\ENDWHILE
\ENDFUNCTION
\end{algorithmic}
\end{algorithm}

In more detail, \dirl iteratively processes vertices in $\G$ starting from the initial vertex $u_0$, continuing until it processes a final vertex $u\in F$. It maintains the property that for every $u$ it processes, it has already trained policies for all edges along some path $\rho_u$ from $u_0$ to $u$.
%at least one path policy $\tilde{\pi}_u$ with final subgoal region $u$.
This property is satisfied by $u_0$ since there is a path of length zero from $u_0$ to itself.
%for the singleton path policy $u_0$ (this policy does not do anything since any state $s\sim\eta$ satisfies $s\in u_0$ by definition of $u_0$).
In Algorithm~\ref{alg:dijkstra}, $\Gamma_u$ is the set of all paths from $u_0$ to $u$ discovered so far, $\Gamma=\bigcup_{u}\Gamma_u$, and $\Pi = \{\pi_e\mid e=u\to u'\in E, u\in U_p\}$ is the set of all edge policies trained so far.
In each iteration, \dirl processes an unprocessed vertex $u$ nearest to $u_0$, which it discovers using \textsc{NearestVertex},  and performs the following steps:
\begin{enumerate}
\item \textsc{ShortestPath} selects the shortest path from $u_0$ to $u$ in $\Gamma_u$, denoted $\rho_u=u_0\to\cdots\to u_k=u$.  
%\begin{align*}
%\rho_u=u_0\to\cdots\to u_k=u.
%\end{align*}
\item \textsc{ReachDistribution} computes the distribution $\eta_u$ over states in $\beta(u)$ induced by using the sequence of policies $\pi_{e_0},...,\pi_{e_{k-1}}\in\Pi$, where $e_j=u_j\to u_{j+1}$ are the edges in $\rho_u$.
\item For every edge $e=u\to u'$, \textsc{LearnPolicy} learns a policy $\pi_e$ for $e$ using $\eta_u$ as the initial state distribution, and adds $\pi_e$ to $\Pi$ and $\rho_{u'}$ to $\Gamma_{u'}$, where
$\rho_{u'}=u_0\to\cdots\to u\to u'$;
$\pi_e$ is trained to ensure that the resulting trajectories from $\beta(u)$ to $\beta(u')$ are in $\traj^e_{\safe}$ with high probability.
\end{enumerate}

\subsection{Definitions and Notation}
\label{sec:definitions}

\textbf{Edge costs.} We begin by defining the edge costs used in Djikstra's algorithm. Given a policy $\pi_e$ for edge $e=u\to u'$, and an initial state distribution $\eta_u$ over the subgoal region $\beta(u)$, the cost $c(e)$ of $e$ is the negative log probability that $\pi_e$ safely transitions the system from $s_0\sim\eta_u$ to $\beta(u')$.
%To formalize this notion, we begin with the following:
First, we say a trajectory $\zeta$ starting at $s_0$ \emph{achieves} an $e$ if it safely reaches $\beta(u')$---formally:
\begin{definition}
\rm
An infinite trajectory $\zeta = s_0\to s_1\to\cdots$ \emph{achieves} edge $e = u\to u'$ in $\G$ (denoted $\zeta\models e$) if (i) $s_0\in \beta(u)$, and (ii) there exists $i$ (constrained to be positive if $u\neq u_0$) such that $s_i\in \beta(u')$ and $\zeta_{0:i}\in\traj_{\safe}^e$; we denote the smallest such $i$ by $i(\zeta,e)$.
\end{definition}
Then, the probability that $\pi$ achieves $e$ from an initial state $s_0\sim\eta_u$ is 
\begin{align*}
P(e;\pi_e,\eta_u)=\Pr_{s_0\sim\eta_u,\zeta\sim\mathcal{D}_{\pi_e,s_0}}[\zeta\models e],
\end{align*}
where $\mathcal{D}_{\pi_e,s_0}$ is the distribution over infinite trajectories induced by using $\pi_e$ from initial state $s_0$. 
%Then, the rewards used for learning are $\mathbbm{1}(\zeta\models e)$. 
%Then, \textsc{LearnPolicy} uses a standard RL algorithm to compute $\pi^*_e = \operatorname*{\arg\max}_{\pi}P(e;\pi, \eta_u)$, where
%Shaped rewards can be used to improve learning; see Appendix~\ref{Ap:shaped_rewards} in the supplement.
%Our algorithm interleaves Djikstra's algorithm with using RL to train policies $\pi_e$. 
%It runs Djikstra's algorithm with
Finally, the cost of edge $e$
%\footnote{Note that we take the negative logarithm since Djikstra's algorithm computes the path with minimum summed edge weights.}
 is
$
c(e)=-\log P(e;\pi_e,\eta_u).
$
Note that $c(e)$ is nonnegative for any edge $e$.

\textbf{Path policies.}
Given edge policies $\Pi$ along with a path
$
\rho=u_0\to u_1\to \cdots\to u_k = u
$
in $\G$, we define a \emph{path policy} ${\pi}_{\rho}$ to navigate from $\beta(u_0)$ to $\beta(u)$. In particular, ${\pi}_{\rho}$ executes $\pi_{u_j\to u_{j+1}}$ (starting from $j=0$) until reaching $\beta(u_{j+1})$, after which it increments $j\gets j+1$ (unless $j=k$). That is, ${\pi}_{\rho}$ is designed to achieve the sequence of edges in $\rho$. Note that ${\pi}_{\rho}$ is stateful since it internally keeps track of the index $j$ of the current policy.

\textbf{Induced distribution.}
Let path $\rho=u_0\to\cdots\to u_k=u$ from $u_0$ to $u$ be such that edge policies for all edges along the path have been trained. 
%Formally, given $\tilde{\pi}_{\rho}$, $\eta_u$ is the probability distribution over $u$ such that for a set of states $S'\subseteq u$, the probability of $S'$ according to $\eta_{\rho}$ is
%\begin{align*}
%\Pr_{s\sim\eta_{\rho}}[s\in S'] = \Pr_{s_0\sim\eta_{\rho'}, \zeta\sim\mathcal{D}_{\pi_e, s_0}}\left[s_{i(\zeta,e)}\in S'\mid \zeta\models e\right].
%\end{align*}
The induced distribution $\eta_\rho$ is defined inductively on the length of $\rho$.
Formally, for the zero length path $\rho = u_0$ (so $u=u_0$), we define $\eta_\rho = \eta$ to be the initial state distribution of the MDP $\M$.
Otherwise, we have $\rho = \rho'\circ e$, where $e = u'\to u$. Then, we define $\eta_\rho$ to be the state distribution over $\beta(u)$ induced by using $\pi_e$ from $s_0\sim\eta_{\rho'}$ conditioned on $\zeta\models e$. Formally, $\eta_\rho$ is the probability distribution over $\beta(u)$ such that for a set of states $S'\subseteq \beta(u)$, the probability of $S'$ according to $\eta_\rho$ is
\begin{align*}
\Pr_{s\sim\eta_\rho}[s\in S'] = \Pr_{s_0\sim\eta_{\rho'}, \zeta\sim\mathcal{D}_{\pi_e, s_0}}\left[s_{i(\zeta,e)}\in S'\mid \zeta\models e\right].
\end{align*}

\textbf{Path costs.}
The cost of a path $\rho=u_0\to\cdots\to u_k=u$ is
$
c(\rho) = -\sum_{j=0}^{k-1}\log P(e_j;\pi_{e_j},\eta_{\rho_{0:j}})
$
where $e_j = u_j\to u_{j+1}$ is the $j$-th edge in $\rho$, and $\rho_{0:j}=u_0\to\cdots \to u_j$ is the $j$-th prefix of $\rho$.
%Note that this definition requires $\eta_{u_i}$ to be well defined for all $i < k$. 
%, and $\rho_{0:i}=u_0\to\cdots\to u_i$ is the length $i$ prefix of $\rho$.

\subsection{Algorithm Details}

\dirl interleaves Djikstra's algorithm with using RL to train policies $\pi_e$. 
Note that the edge weights to run Dijkstra's are not given \emph{a priori} since the edge policies and initial state/induced distributions are unknown. Instead, they are computed on-the-fly beginning from the subgoal region $u_0$ using Algorithm~\ref{alg:dijkstra}. 
We describe each subprocedure below.

\textbf{Processing order (\textsc{NearestVertex}).}
On each iteration, \dirl chooses the vertex $u$ to process next to be an unprocessed vertex that has the shortest path from $u_0$---i.e.,
$
u \in \operatorname*{\arg\min}_{u'\in U\setminus U_p}\operatorname*{\min}_{\rho\in\Gamma_{u'}}c(\rho).
$
This choice is an important part of Djikstra's algorithm. For a graph with fixed costs, it ensures that the computed path $\rho_u$ to each vertex $u$ is minimized. While the costs in our setting are not fixed since they depend on $\eta_u$, this strategy remains an effective heuristic.

\textbf{Shortest path computation
(\textsc{ShortestPath}).}
%
%We define the cost of a path $\rho=u_0\to\cdots\to u_k=u$ as follows. Let $\rho_{i} = u_0\to\cdots\to u_i$ denote the prefix of $\rho$ of length $i$. Then, the cost of the path $\rho$ is
This subroutine returns a path of minimum cost,
$
\rho_u \in \operatorname*{\arg\min}_{\rho\in\Gamma_u}c(\rho).
$
These costs can be estimated using Monte Carlo sampling. 

\textbf{Initial state distribution
(\textsc{ReachDistribution}).}
A key choice \dirl makes is what initial state distribution $\eta_u$ to choose to train policies $\pi_e$ for outgoing edges $e=u\to u'$. 
\dirl chooses the initial state distribution $\eta_u=\eta_{\rho_u}$ to be the distribution of states reached by the path policy ${\pi}_{\rho_u}$ from a random initial state $s_0\sim\eta$.\footnote{This choice is the distribution of states reaching $u$ by the path policy ${\pi}_{\rho}$ eventually returned by \dirl. Thus, it ensures that the training and test distributions for edge policies in ${\pi}_{\rho}$ are equal.}

%This is the initial distribution used to train policies $\pi_e$ for outgoing edges $e  = u\rightarrow u'$.

\textbf{Learning an edge policy 
(\textsc{LearnPolicy}).}
%\textbf{\textsc{LearnPolicy}.}
%\label{sec:learning}
%\textbf{Learning a single edge.}
%
%We have one neural network policy for every edge $(u_1,u_2)$ in the abstract graph which is to be trained to transitions the system from any state in the source subgoal region $u_1$ to the target subgoal region $u_2$.
Now that the initial state distribution $\eta_u$ is known, we describe how \dirl learns a policy $\pi_e$ for a single edge $e=u\to u'$. At a high level, it trains $\pi_e$ using a standard RL algorithm, where the rewards $\mathbbm{1}(\zeta\models e)$ are designed to encourage $\pi_e$ to safely transition the system to a state in $\beta(u')$. To be precise, \dirl uses RL to compute
%$\mathbbm{1}(s\in u')$ {\color{red} This should be $\mathbbm{1}(\zeta\models e)$ since safety is also needed}.
$
\pi_e \in \operatorname*{\arg\max}_{\pi}P(e;\pi, \eta_u).
$
Shaped rewards can be used to improve learning; see \autoref{Ap:shaped_rewards}.

\textbf{Constructing a path policy
(\textsc{PathPolicy}).}
Given edge policies $\Pi$ along with a path $\rho=u_0\to\cdots\to u$, where $u\in F$ is a final vertex, \dirl returns the path policy ${\pi}_{\rho}$.

% For the sake of clarity, we have omitted details on how \dirl learns to maintain safety after reaching a final region $u$. This is achieved by training a separate policy to be used after reaching $u$. Details can be found in \autoref{Ap:FinalPolicy}.
%\begin{align*}
%\rho=u_0\to u_1\to \cdots\to u_k = u,
%\end{align*}
%in $\G$, our algorithm constructs a \emph{path policy} $\tilde{\pi}_{\rho}$ to navigate from $u_0$ to $u$. In particular, $\tilde{\pi}_{\rho}$ executes $\pi_{u_i\to u_{i+1}}$ (starting from $i=0$) until reaching $u_{i+1}$, after which it increments $i\gets i+1$ (unless $i=k$). That is, $\tilde{\pi}_{\rho}$ is designed to achieve the sequence of edges in $\rho$. Note that $\tilde{\pi}_{\rho}$ is stateful since it internally keeps track of the index $i$ of the current policy.

\textbf{Theoretical Guarantee.}
We have the following guarantee (we give a proof in Appendix~\ref{sec:proofs}).
\begin{theorem}
\label{thm:main}
Given a path policy ${\pi_{\rho}}$ corresponding to a path $\rho = u_0\to\cdots\to u_k = u$, where $u\in F$, we have
$\Pr_{\zeta\sim\D_{{\pi_{\rho}}}}[\zeta \models \G]\geq\exp(-c(\rho))$.
%\prod_{i=0}^{k-1}P(e_i;\pi_{e_i},\eta_{\rho_i})
%%\cdot P(u_k;\pi_k,\eta_k).
%where $e_i = u_i\to u_{i+1}$.
\end{theorem}
In other words, we guarantee that minimizing the path cost $c(\rho)$ corresponds to maximizing a lower bound on the objective of the abstract reachability problem.

\section{Experiments}\label{sec:exp}
We empirically evaluate our approach on several continuous control environments; details are in \autoref{ap:methodology}, \ref{Ap:RoomsCaseStudy} and \ref{Ap:FetchCaseStudy}. 

\textbf{Rooms environment.}
We consider the 9-Rooms environment shown in \autoref{fig:Motivate}, and a similar 16-Rooms environment. They have states $(x,y)\in\mathbb{R}^2$ encoding 2D position, actions $(v,\theta)\in\mathbb{R}^2$ encoding speed and direction, and transitions $s'=s+(v\cos(\theta),v\sin(\theta))$.
% One of the 16-Rooms environment is visualized in \autoref{fig:Rooms}. The other is more constrained as several doors are shut. 
% \begin{figure}
% \centering
% \includegraphics[width=0.8\linewidth,page=4,trim={0cm 8cm 23cm 0cm}]{Figures/Rooms.pdf}
% \caption{16-Rooms (All doors open). The red blocks are obstacles. Blue box indicates the initial room.}
% \label{fig:Rooms}
% \end{figure}
For 9-Rooms, we consider specifications similar to $\p_{\text{ex}}$ in \autoref{fig:Motivate}. 
%Recall, in those sequential specifications, the task is to choose between visiting two locations before proceeding to the final room. The specifications are designed in a way that a greedy choice in the first sub-task will be unfavorable to complete the specification. 
For 16-Rooms, we consider a series of increasingly challenging specifications $\p_1,...,\p_5$; each $\p_i$ encodes a sequence of $i$ sub-specifications, each of which has the same form as $\p_{\text{ex}}$ (see \autoref{Ap:RoomsCaseStudy}).
%We design a suite of scalable benchmarks for the 16-Rooms environment. 
We learn policies using ARS~\citep{mania2018simple} with shaped rewards (see \autoref{Ap:shaped_rewards}); each one is a fully connected NN with 2 hidden layers of 30 neurons each.

%We consider two environments, 9-Rooms and 16-Rooms. Some of these rooms maybe connected through bi-directional doors in the thick walls. 
%They have states $(x,y)\in\R^2$ encoding 2D position, actions $(v,\theta)\in\R^2$ encoding speed and direction, and transitions $s_0 = s + (v \cos{\theta}, v \sin\theta)$. 

\textbf{Fetch environment.}
We consider the Fetch-Pick-And-Place environment in OpenAI Gym \citep{gym}, consisting of a robotic arm that can grasp objects and a block to manipulate. The state space is $\R^{25}$, which includes components encoding the gripper position, the (relative) position of the object, and the distance between the gripper fingers. The action space is $\mathbb{R}^4$, where the first 3 components encode the target gripper position and the last encodes the target gripper width. The block's initial position is a random location on a table. We consider predicates \emph{NearObj} (indicates if the gripper of the arm is close to the block), \emph{HoldingObj} (indicates if the gripper is holding the block), \emph{LiftedObj} (indicates if the block is above the table), and \emph{ObjAt}$[g]$ (indicates if the block is close to a goal $g\in\R^3$).

We consider three specifications. First, \emph{PickAndPlace} is
\begin{align*}
\p_1 = \text{NearObj}; {\text{HoldingObj}}; {\text{LiftedObj}}; {\text{ObjAt}[g]},
\end{align*}
where $g$ is a random goal location. Second, \emph{PickAndPlaceStatic} is similar to the previous one, except the goal location is fixed. Third, \emph{PickAndPlaceChoice} involves choosing between two tasks, each of which is a sequence of two subtasks similar to PickAndPlaceStatic.
%The final goal region of one task is significantly smaller than the other making high level planning important.
%(we abuse notation to denote $\eventually{b}$ using $b$).
%\begin{itemize}[topsep=0pt,itemsep=0ex,partopsep=1ex,parsep=1ex]
%\item \textbf{PickAndPlace}. This is the standard task of the Gym environment. The specification we use is $\p_1 = $ {\text{GripNearObject}}; {\text{HoldingObject}}; {\text{LiftedObject}}; {\text{ObjectAt}($g$)} where $g$ is a random goal location that is part of the observation.
%\item \textbf{PickAndPlaceStatic}. This is similar to the previous task except that the goal location is a fixed location at a corner of the table.
%\item \textbf{PickAndPlaceChoice}. This task involves choosing between two tasks each of which involves a sequence of two goals (after lifting the object). The final goal region of one task is significantly smaller than the other making high level planning important.
%\end{itemize}
We learn policies using TD3~\citep{fujimoto2018addressing} with shaped rewards; each one is a fully connected NN with 2 hidden layers of 256 neurons each.

\textbf{Baselines.}
We compare our approach to four state-of-the-art algorithms for learning from specifications, \toolname~\citep{spectrl}, \textsc{Qrm}~\citep{icarte2018using}, \textsc{Hrm}~\citep{tor-etal-arxiv20}, and a \tltl~\citep{li2017reinforcement} based approach, as well as a state-of-the-art hierarchical RL algorithm, \textsc{R-avi}~\citep{jothimurugan2021abstract}, that leverages state abstractions. We used publicly available implementations of \toolname, \textsc{Qrm}, \textsc{Hrm} and \textsc{R-avi}. For \textsc{Qrm} and \textsc{Hrm}, we manually encoded the tasks as reward machines with continuous rewards. The variants \textsc{Qrm+cr} and \textsc{Hrm+cr} use counterfactual reasoning to reuse samples during training. Our implementation of \tltl uses the quantitative semantics defined in \citet{li2017reinforcement} with ARS to learn a single policy for each task. We used the subgoal regions and the abstract graph generated by our algorithm as inputs to \textsc{R-avi}. Since \textsc{R-avi} only supports disjoint subgoal regions and furthermore assumes the ability to sample from any subgoal region, we only ran \textsc{R-avi} on supported benchmarks. The learning curves for \textsc{R-avi} denote the probability of reaching the final goal region in the $y$-axis which is an upper bound on the probability of satisfying the specification. Note that \dirl returns a policy only after the search finishes. Thus, to plot learning curves, we ran our algorithm multiple times with different number of episodes used for learning edge policies. 

%We note that \tltl had difficulty scaling to some of the larger benchmarks, so we omit results for them. As discussed below, \dirl significantly outperforms \tltl even on the easier benchmarks.

%For a fair comparison with the baselines, if each episode of our tool on an edge is run with $m$ steps, we run the episodes of the baselines with $m\cdot p + c$ steps, where $p$ is the minimum path length to reach a goal state in the abstract graph of the specification and $c>0$ is a buffer. Intuitively, this ensures that all tools get a similar number of steps in each episode to learn the specification. 

\begin{figure*}[t]
\centering
\begin{subfigure}{0.3\textwidth}
\centering
\includegraphics[width=\textwidth]{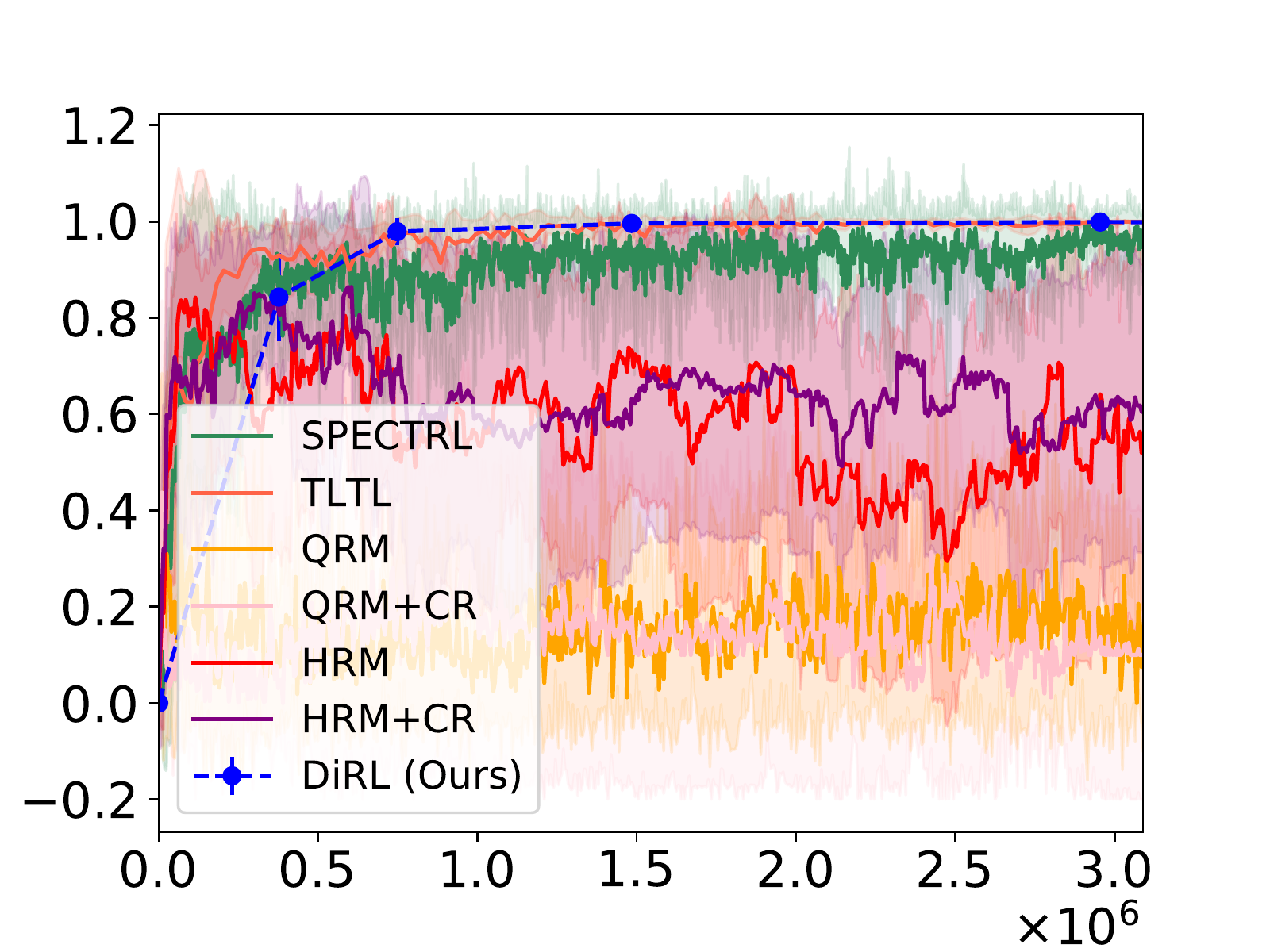}
\caption{Half sub-specification $\p_1$, $|\G_{\p_1}| = 2$.}
\label{fig:16rooms9}
\end{subfigure}
\quad
%\hfill
\begin{subfigure}{0.3\textwidth}
\centering
\includegraphics[width=\textwidth]{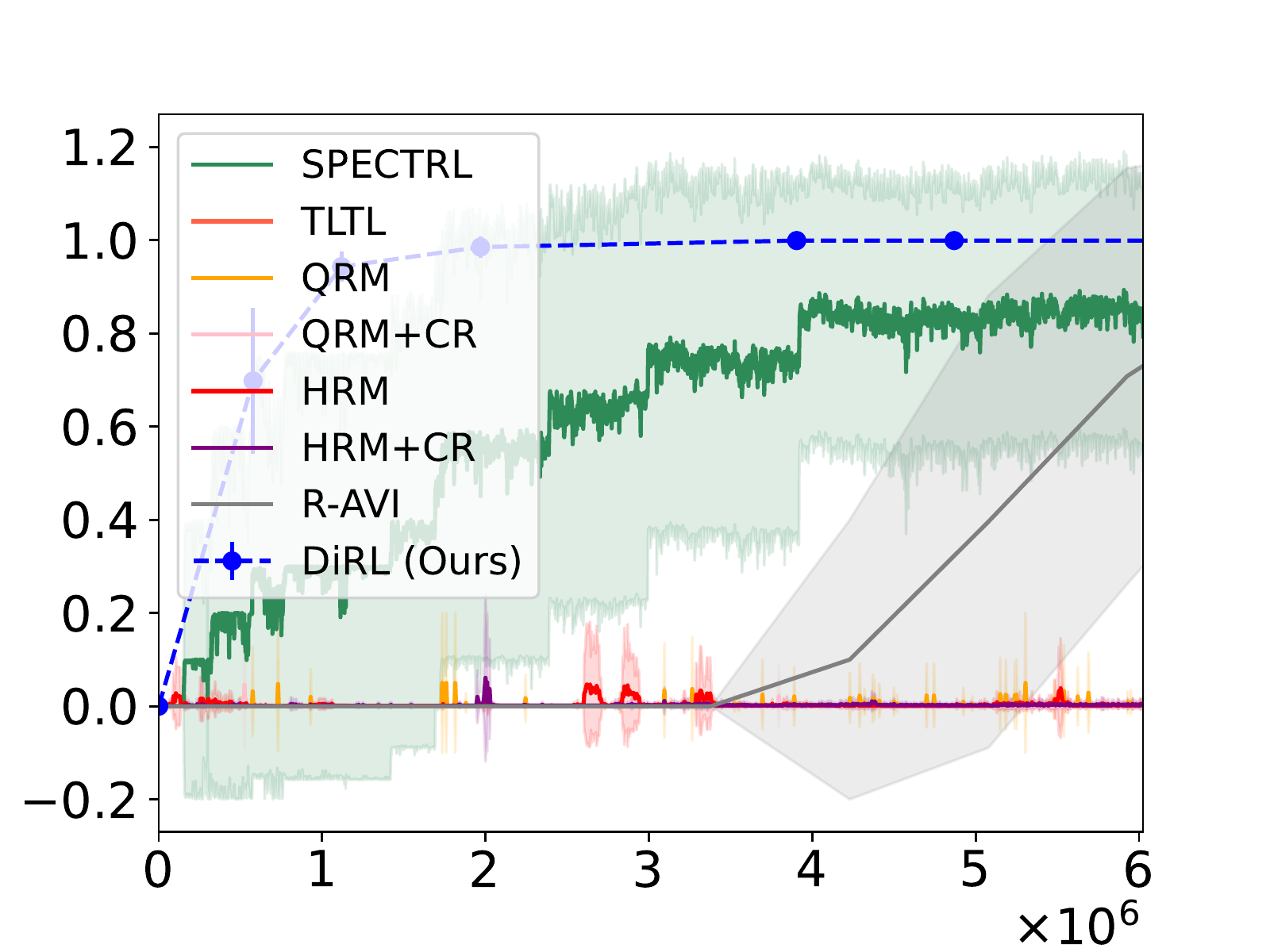}
\caption{1 sub-specification $\p_2$, $|\G_{\p_2}| = 4$.}
\label{fig:16rooms10}
\end{subfigure}
\quad
%\hfill
\begin{subfigure}{0.3\textwidth}
\centering
\includegraphics[width=\textwidth]{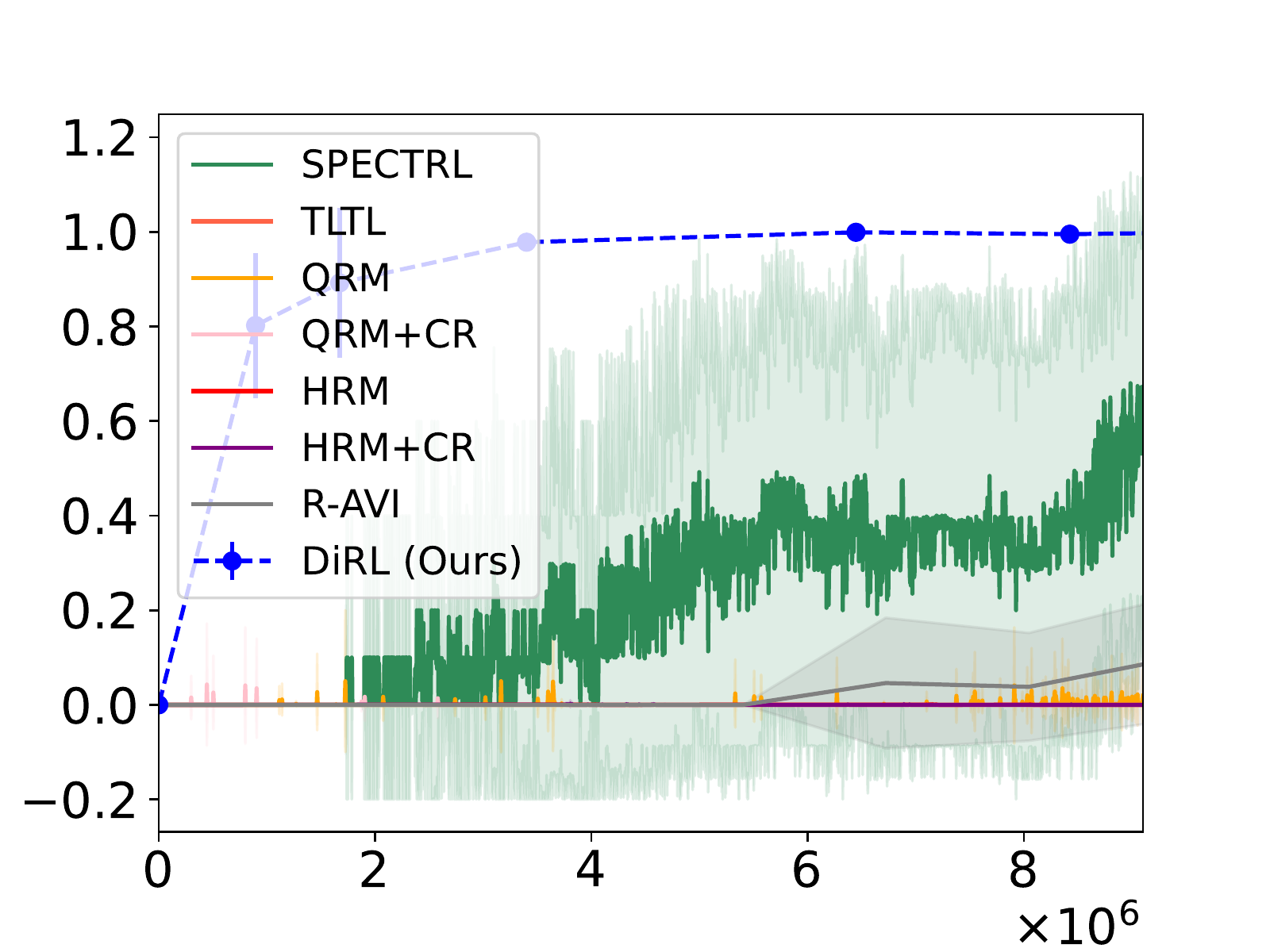}
\caption{2 sub-specifications $\p_3$, $|\G_{\p_3}| = 8$.}
\label{fig:16rooms11}
\end{subfigure}
%\hfill         
\begin{subfigure}{0.3\textwidth}
\centering
\includegraphics[width=\textwidth]{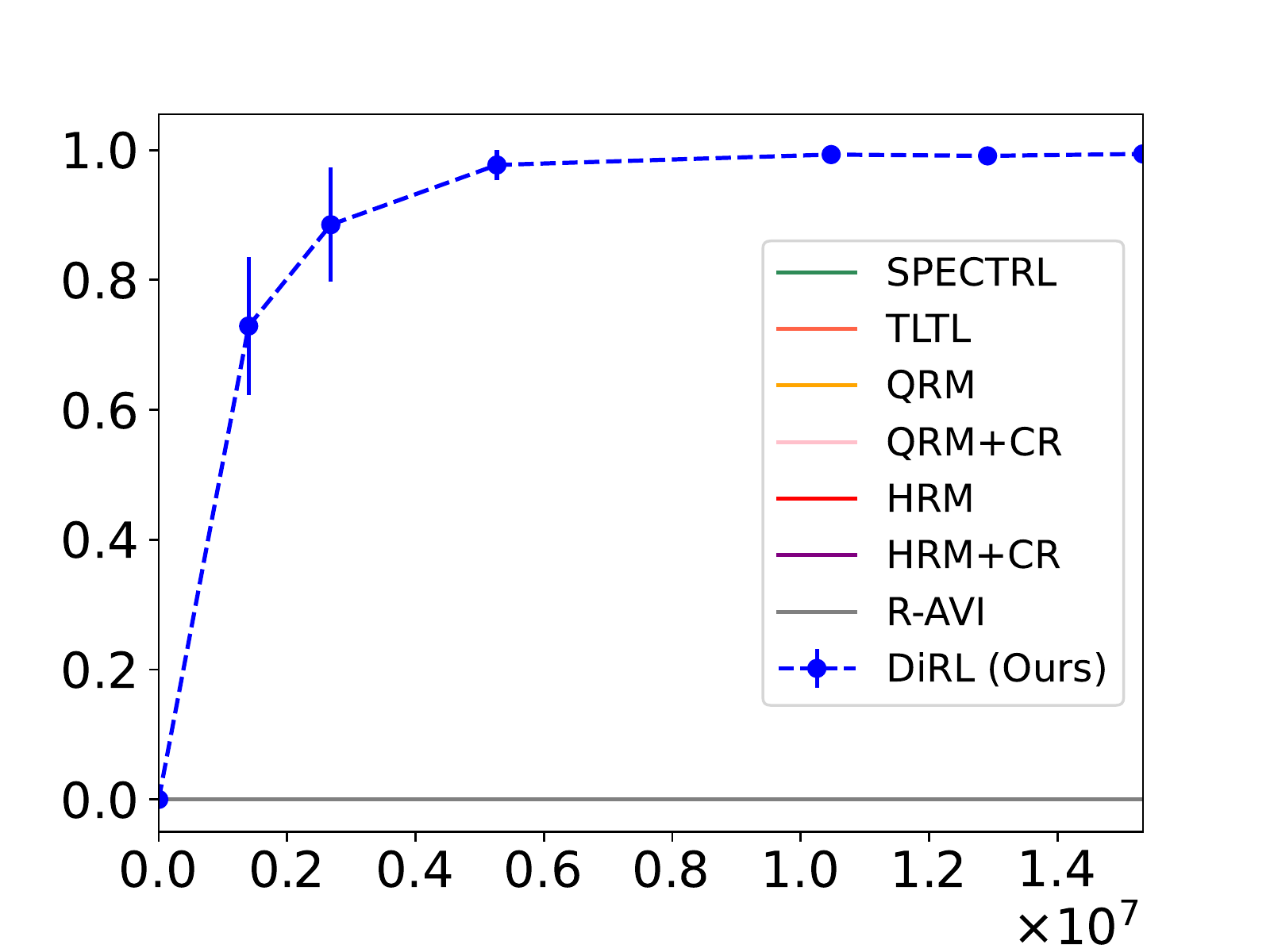}
\caption{3 sub-specifications $\p_4$, $|\G_{\p_4}| = 12$.}
\label{fig:16rooms12}
\end{subfigure}
\quad
%\hfill          
\begin{subfigure}{0.3\textwidth}
\centering
\includegraphics[width=\textwidth]{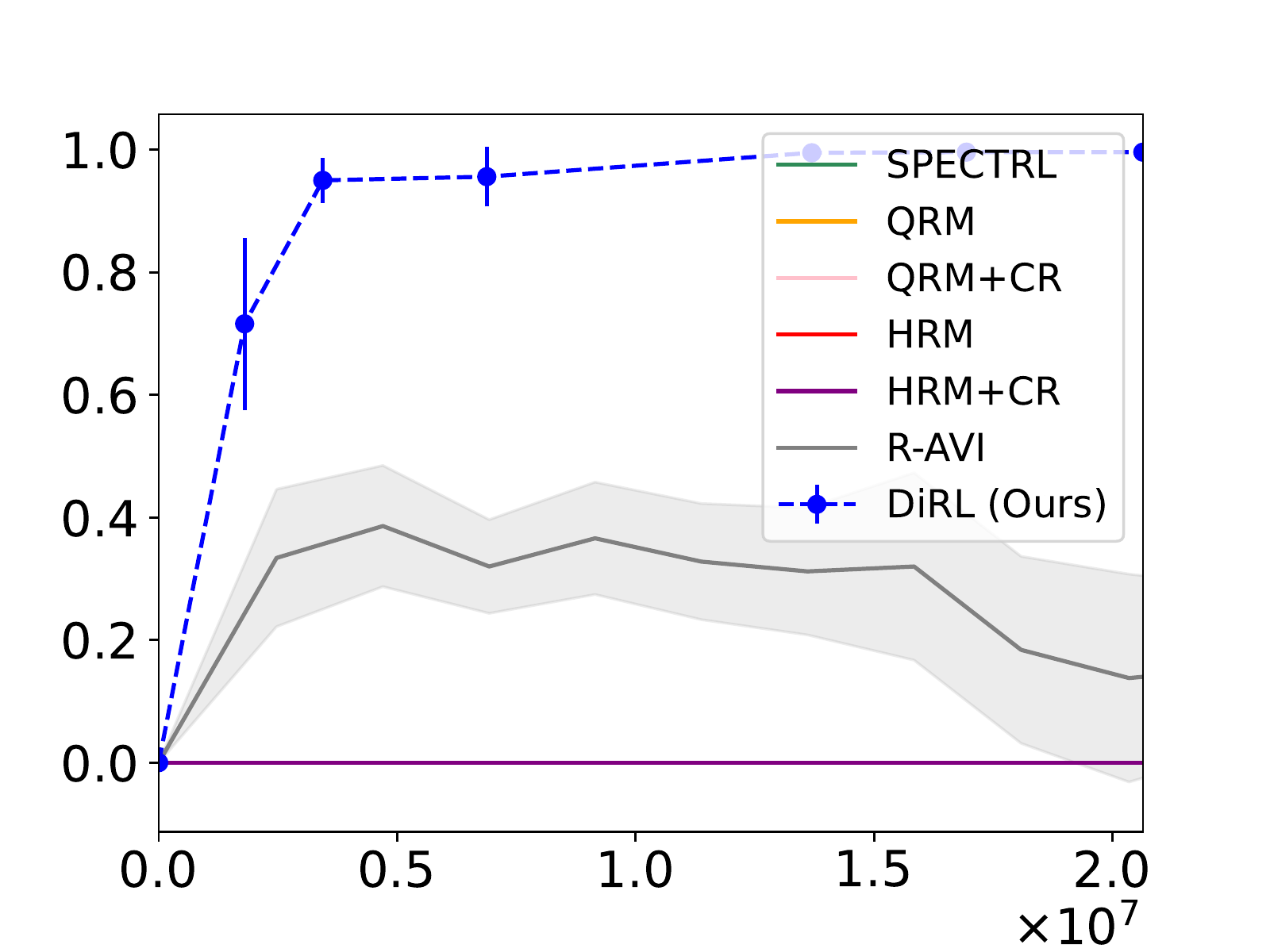}
\caption{4 sub-specifications $\p_5$, $|\G_{\p_5}| = 16$.}
\label{fig:16rooms13}
\end{subfigure}
\quad
\begin{subfigure}{0.3\textwidth}
\centering
\includegraphics[width=\linewidth]{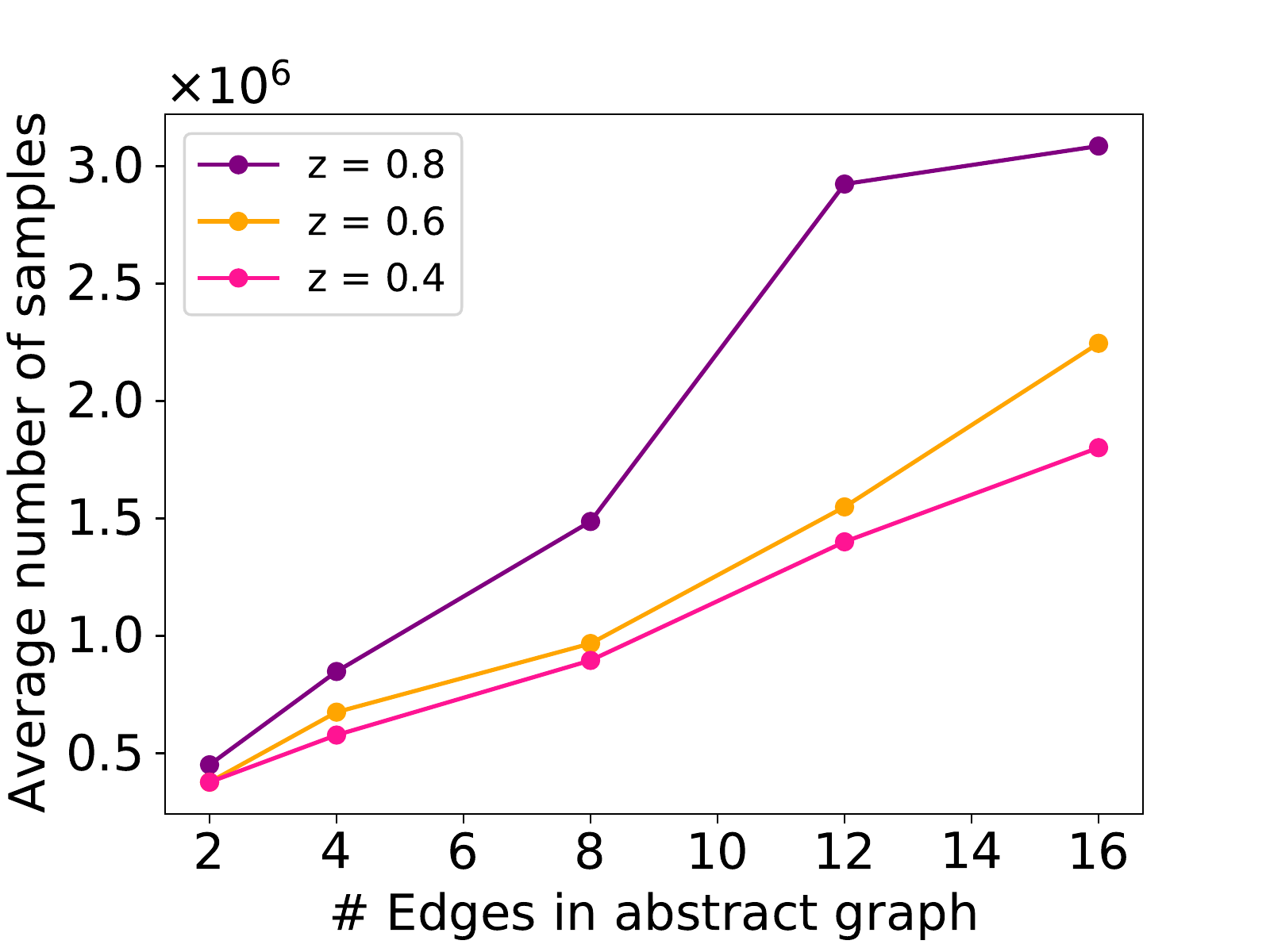}
\caption{Sample complexity curves.}
\label{fig:scalability}
\end{subfigure}
\caption{(a)-(e) Learning curves for 16-Rooms environment with different specifications increasing in complexity from from (a) to (e). $x$-axis denotes the number of samples (steps) and $y$-axis denotes the estimated probability of success.
%Subcaptions correspond to the total number of edges and the minimum path length to reach the goal state in the abstract graph of the specification.
Results are averaged over 10 runs with error bars indicating $\pm$ standard deviation. (f) shows the average number of samples (steps) needed to achieve a success probability $\geq z$ ($y$-axis) as a function of the size of the abstract graph $|\G_{\p}|$.}
\label{Fig:16Rooms}
\end{figure*}

\textbf{Results.}
%
%\autoref{fig:rooms9greedy} plots the learning curve on the 9-Rooms environment. It demonstrates the merits of our Dijkstra-based algorithm as the interleaving of high-level and low-level policies encourages the RL-agent to explore future edges in the abstract graph before committing to a policy for each edge. 
%The experiments comprehensively demonstrate the scalability of our tool. 
\autoref{Fig:16Rooms} shows learning curves on the specifications for 16-Rooms environment with all doors open. None of the baselines scale beyond $\phi_2$ (one segment), while \dirl quickly converges to high-quality policies for all specifications. 
The \tltl baseline performs poorly since most of these tasks require stateful policies, which it does not support. Though \toolname can learn stateful policies, it scales poorly since (i) it does not decompose the learning problem into simpler ones, and (ii) it does not integrate model-based planning at the high-level. Reward Machine based approaches (\textsc{Qrm} and \textsc{Hrm}) are also unable to handle complex specifications, likely because they are completely based on model-free RL, and do not employ model-based planning at the high-level. Although \textsc{R-avi} uses model-based planning at the high-level in conjunction with low-level RL, it does not scale to complex specifications since it trains all edge policies multiple times (across multiple iterations) with different initial state distributions; in contrast, our approach trains any edge policy at most once.

We summarize the scalability of \dirl in \autoref{fig:scalability}, where we show the average number of steps needed to achieve a given success probability $z$ as a function of the number of edges in $\G_\p$ (denoted by $|\G_\p|$). As can be seen, the sample complexity of \dirl scales roughly linearly in the graph size. Intuitively, each subtask takes a constant number of steps to learn, so the total number of steps required is proportional to $|\G_\p|$.
%The scalability is primarily attributed to the decomposition into sub-tasks, each of which is easy to learn. Therefore, since our algorithm learns policy for each sub-task, it scales well in the number of edges in the abstract graph, as  illustrated in 
%The learning curve for the 9-Rooms environment with the specification presented as the motivating example is shown in \autoref{fig:rooms9greedy}.
In the supplement, we show learning curves for 9-Rooms (\autoref{Fig:9RoomsLC}) for a variety of specifications, and learning curves for a variant of 16-Rooms with many blocked doors with the same specifications described above (\autoref{Fig:16_4Rooms}). These experiments demonstrate the robustness of our tool on different specifications and environments.
For instance, in the 16-Rooms environment with blocked doors, fewer policies satisfy the specification, which makes learning more challenging but \dirl is still able to learn high-quality policies for all the specifications.
%For instance, we observed \toolname converges slower on the specification with one step on the constrained 16-Rooms environment with fewer open doors than on the unconstrained one with all doors open.
%These difficulties make the baselines converge more slowly, but because \dirl leverages the structure of the specification, its performance is unaffected. 
%by these changes in the environment since it conducts the same number of operations to explore the abstract graph.

\begin{figure*}[t]
\centering
\begin{subfigure}{0.3\textwidth}
\centering
\includegraphics[width=\textwidth]{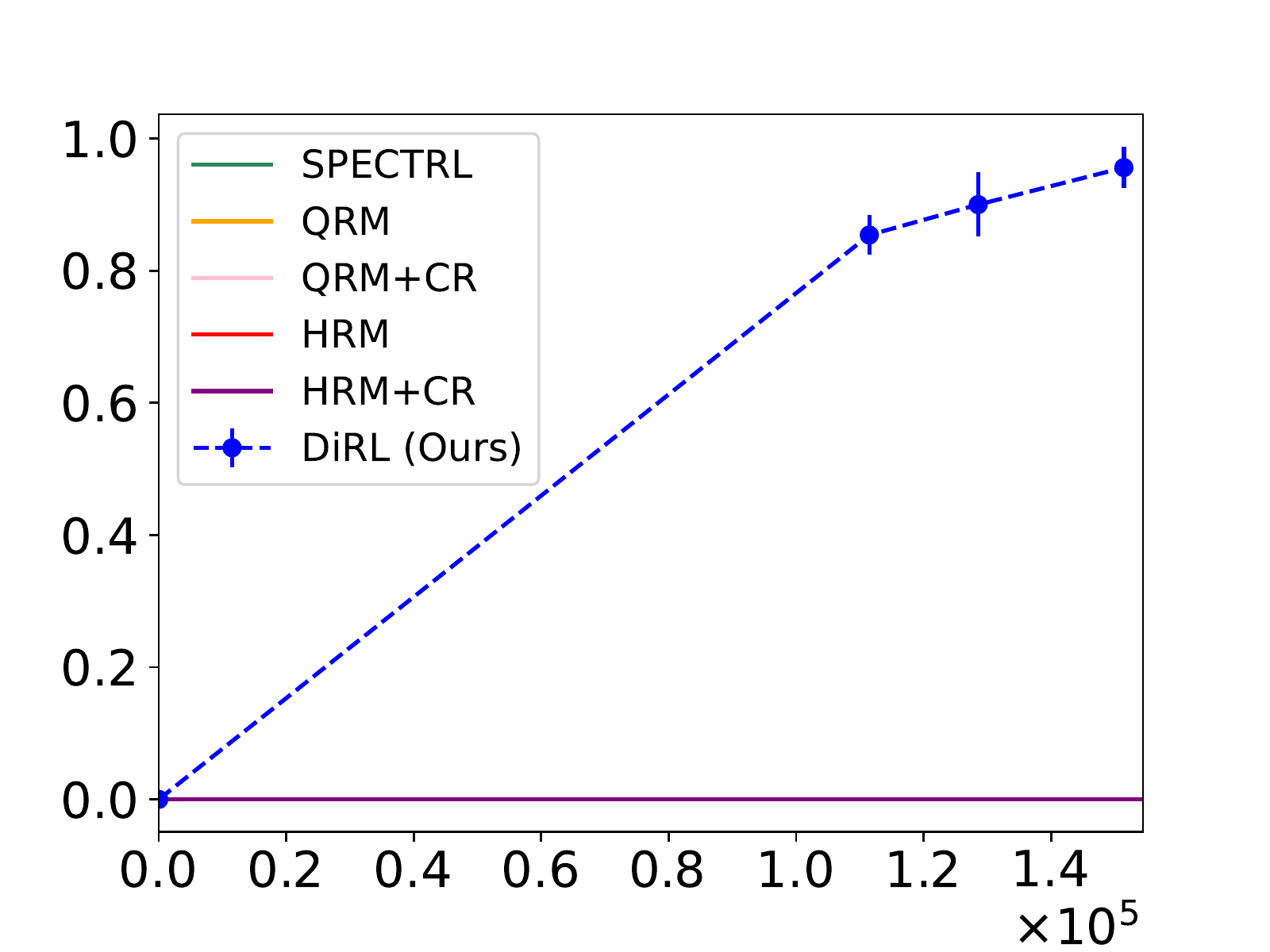}
\caption{PickAndPlace}
\label{fig:fpp_spec3}
\end{subfigure}
\quad
\begin{subfigure}{0.3\textwidth}
\centering
\includegraphics[width=\textwidth]{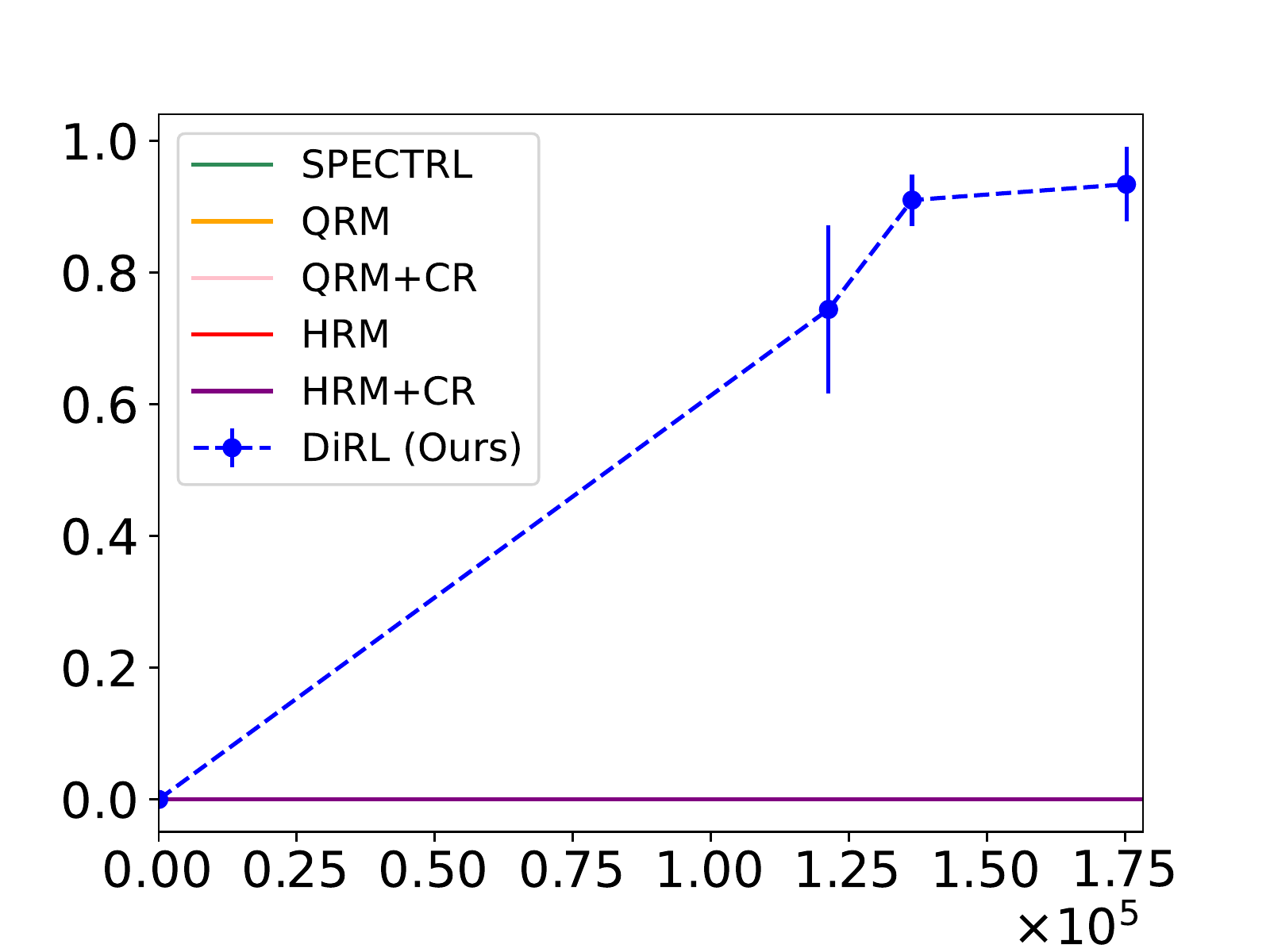}
\caption{PickAndPlaceStatic}
\label{fig:fpp_spec4}
\end{subfigure}
\quad
\begin{subfigure}{0.3\textwidth}
\centering
\includegraphics[width=\textwidth]{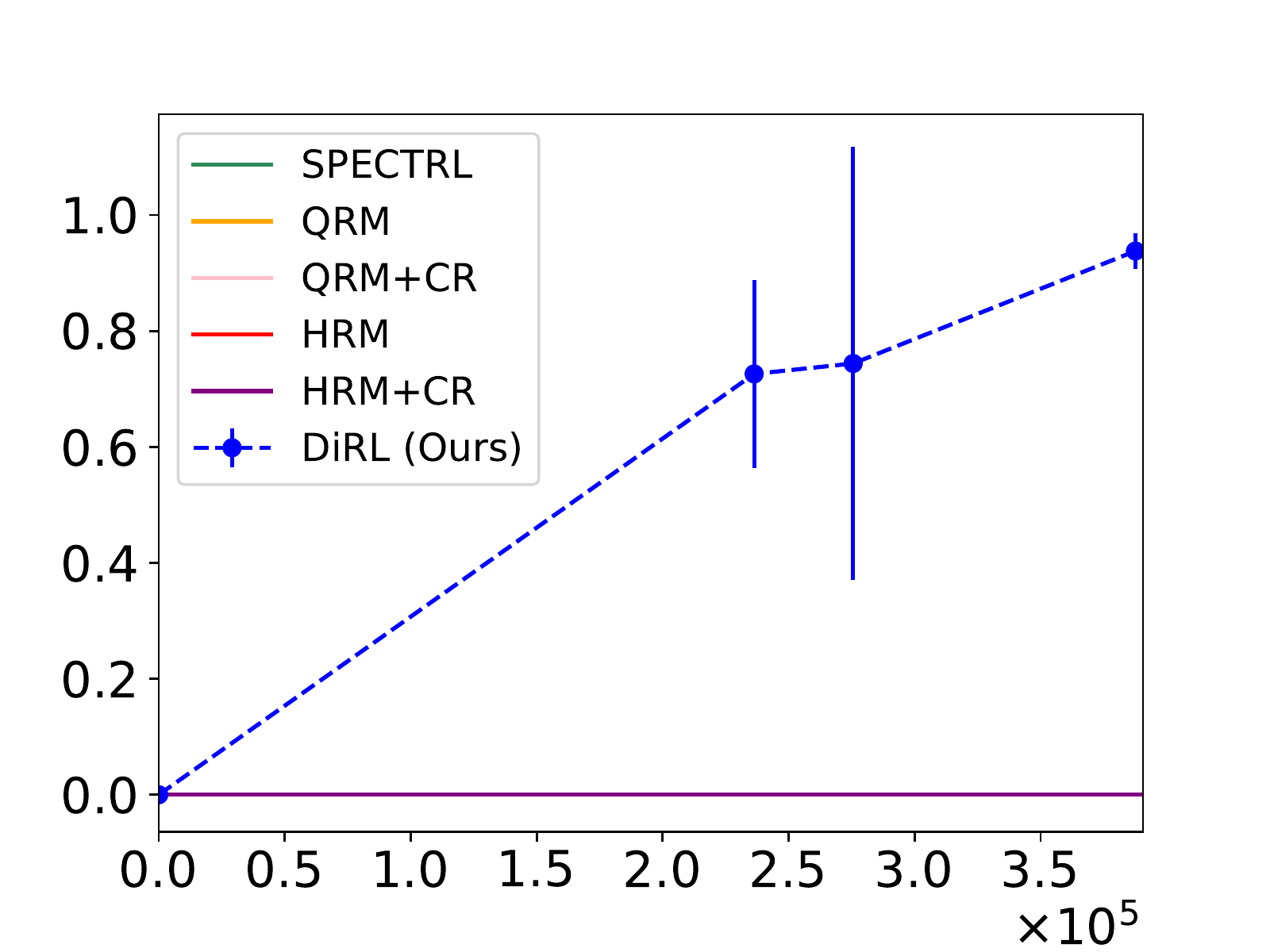}
\caption{PickAndPlaceChoice}
\label{fig:fpp_spec5}
\end{subfigure}
\caption{Learning curves for Fetch environment; $x$-axis denotes the total number of samples (steps) and $y$-axis denotes the estimated probability of success. Results are averaged over 5 runs with error bars indicating $\pm$ standard deviation.}
\label{Fig:fpp}
\end{figure*}

Next, we show results for the Fetch environment in \autoref{Fig:fpp}. The trends are similar to before---\dirl leverages compositionality to quickly learn effective policies, whereas the baselines are ineffective. The last task is especially challenging, taking \dirl somewhat longer to solve, but it ultimately achieves similar effectiveness. These results demonstrate that \dirl can scale to complex specifications even in challenging environments with high-dimensional state spaces.

\section{Conclusions}\label{sec:conc}
%\textbf{Sources of suboptimality.}
%
%While we build on Djikstra's algorithm, which is guaranteed to find the optimal path, these guarantees do not translate to our setting. However, Djikstra's algorithm assumes the edge costs $c(e)$ are constant. In our framework, they depend on the choice of $\eta_u$. We have made a specific choice; intuitively, our choice is to greedily use the initial state distribution of the lowest cost path to $u$. However, other strategies for choosing $\eta_u$ may result in path policies that are overall more optimal.

%There are two sources of suboptimality in our approach. Firstly, we are restricting the class of policies to be of a particular kind, namely, a sequence of policies corresponding to a fixed path in the abstract graph; it is possible that the optimal policy chooses different edges from different concrete states with the same subgoal region. Secondly, when we learn a policy for an edge $u_1\to u_2$, we take the initial state distribution to be the one defined by the current best path to $u_1$ which might not necessarily lead to optimal paths in the future; this is a greedy heuristic that we show works well in our experiments.

%Another issue is that we are restricting to path policies---i.e., sequences of policies corresponding to a fixed path in the abstract graph. In general, it is possible that the optimal policy chooses different edges from different concrete states in the same subgoal region.

We have proposed \dirl, a reinforcement learning approach for logical specifications that leverages the compositional structure of the specification to decouple high-level planning and low-level control. Our experiments demonstrate that \dirl can effectively solve complex continuous control tasks, significantly improving over existing approaches. Logical specifications are a promising approach to enable users to more effectively specify robotics tasks; by enabling more scalable learning of these specifications, we are directly enabling users to specify more complex objectives through the underlying specification language. While we have focused on \toolname specifications, we believe our approach can also enable the incorporation of more sophisticated features into the underlying language, such as conditionals (i.e., only perform a subtask upon observing some property of the environment) and iterations (i.e., repeat a subtask until some objective is met).

\subsection*{Limitations}  \dirl assumes the ability to sample trajectories starting at any state $s\in S$ that has been observed before, whereas in some cases it might only be possible to obtain trajectories starting at some initial state. One way to overcome this limitation is to use the learnt path policies for sampling---i.e., in order to sample a state from a subgoal region $\beta(u)$ corresponding to a vertex $u$ in the abstract graph, we could sample an initial state $s_0\sim\eta$ from $\beta(u_0)$ and execute the path policy $\pi_{\rho_u}$ corresponding to the shortest path $\rho_u$ from $u_0$ to $u$ starting at $s_0$. Upon successfully reaching $\beta(u)$ (we can restart the sampling procedure if $\beta(u)$ is not reached), the system will be in a state $s\sim\eta_u$ in $\beta(u)$ from which we can simulate the system further.

Another limitation of our approach is that we only consider path policies. It is possible that an optimal policy must follow different high-level plans from different states within the same subgoal region. We believe this limitation can be addressed in future work by modifying our algorithm appropriately.

\subsection*{Societal impacts} Our work seeks to improve reinforcement learning for complex long-horizon tasks. Any progress in this direction would enable robotics applications both with positive impact---e.g., flexible and general-purpose manufacturing robotics, robots for achieving agricultural tasks, and robots that can be used to perform household chores---and with negative or controversial impact---e.g., military applications. These issues are inherent in all work seeking to improve the abilities of robots.

\subsection*{Acknowledgements and Funding} We thank the anonymous reviewers for their helpful comments. Funding in direct support of this work: CRA/NSF Computing Innovations Fellow Award,  DARPA Assured Autonomy project under Contract No. FA8750-18-C-0090, ONR award N00014-20-1-2115, NSF grant CCF-1910769, and ARO grant W911NF-20-1-0080.

\bibliography{main}
\bibliographystyle{plainnat}

%%%%%%%%%%%%%%%%%%%%%%%%%%%%%%%%%%%%%%%%%%%%%%%%%%%%%%%%%%%%

\appendix
% \section*{Checklist}
% \input{checklist}
\clearpage
\setcounter{page}{1}
%Appendix 

\section{Reduction to Abstract Reachability}\label{sec:reduction}
In this section, we detail the construction of the abstract graph $\G_{\p}$ from a \toolname specification $\p$. Given two sets of finite trajectories $\traj_1,\traj_2\subseteq\traj_f$, let us denote by $\traj_1\circ\traj_2$ the concatenation of the two sets---i.e.,
\begin{align*}
\traj_{1}\circ\traj_{2} =
\left\{\zeta\in\traj_f \biggm\vert
\begin{array}{c}
\exists i< t\;.\; \zeta_{0:i}\in\traj_{1} \\
\wedge \; \zeta_{(i+1):t}\in\traj_{2}
\end{array}\right\}.
\end{align*}
In addition to the abstract graph $\G = (U, E, u_0, F,\beta,\traj_\safe)$ we also construct a set of safe terminal trajectories $\traj_{\term} = \bigcup_{u\in F}\traj_{\term}^u$ where $\traj_{\term}^u\subseteq\traj_f$ is the set of terminal trajectories for the final vertex $u\in F$. Now, we define what it means for a finite trajectory $\zeta$ to satisfy the pair $(\G, \traj_\term)$.

\begin{definition}
\rm
A finite trajectory $\zeta=s_0\xrightarrow{a_0}s_1\xrightarrow{a_1}\cdots\xrightarrow{a_{t-1}}s_t$ in $\M$ satisfies the pair $(\G,\traj_\term)$ (denoted $\zeta\models (\G,\traj_\term)$) if there is a sequence of indices $0=i_0\leq i_1<\cdots<i_k\leq t$ and a path $\rho=u_0\to u_1\to\cdots\to u_k$ in $\G$ such that
\begin{itemize}[topsep=0pt,itemsep=0ex,partopsep=1ex,parsep=1ex]
\item $u_k\in F$,
\item for all $j\in\{0,\ldots,k\}$, we have $s_{i_j}\in \beta(u_j)$,
\item for all $j < k$, letting $e_j=u_j\to u_{j+1}$, we have $\zeta_{i_j:i_{j+1}}\in\traj_{\safe}^{e_j}$, and
\item $\zeta_{i_k:t}\in\traj_\term^{u_k}$.
\end{itemize}
\end{definition}

We now outline the inductive construction of the pair $(\G_\p,\traj_{\term,\p})$ from a specification $\p$ such that any finite trajectory $\zeta\in\traj_f$ satisfies $\p$ if and only if $\zeta$ satisfies $(\G_\p,\traj_{\term,\p})$.

% Recall that for any predicate $b$ we denote by $u_b$ the set of states at which $b$ holds---i.e, $u_b = \{s\mid s\models b\}$. For a set of states $V \subseteq S$, we denote $\traj(V)$ the set of trajectories that always remain in $V$, i.e, $\traj(V) = \{\zeta = s_0\xrightarrow{a_0}\ldots \xrightarrow{a_{t-1}}s_t \mid \forall i\in\N, s_i \in V\}$.

\textbf{Objectives} ($\p = \eventually{b}$).
% The task monitor is
% $$\A_{\p}=(\{q_0,q_1\},\Delta,\{q_1\},q_{0})$$
% where $\Delta = \{(q_0, \code{True},q_0), (q_0, b, q_1), (q_1, \code{True}, q_1)\}$.
The abstract graph is $\G_{\p} = (U, E, u_0, F, \beta,\traj_{\safe})$ where
\begin{itemize}
\item $U = \{u_0, u_b\}$ with $\beta(u_0)=S$ and $\beta(u_b) = S_b = \{s\mid s\models b\}$,
\item $E = \{u_0\to u_b\}$,
\item $F = \{u_b\}$ and,
\item $\traj_{\safe}^{(u_0,u_b)} = \traj_{\term}^{u_b} = \traj_f$.
\end{itemize}

\textbf{Constraints} ($\p = \p_1\always{b}$). Let
%the monitor for $\p_1$ be given by $\A_{\p_1} = (Q_1,\Delta_1, Q_F^1,q_{0}^1)$ and
the abstract graph for $\p_1$ be $\G_{\p_1} = (U_1,E_1,u_{0}^1,F_1,\beta_1,\traj_{\safe,1})$ and the terminal trajectories be $\traj_{\term, 1}$. Then, the abstract graph for $\p$ is $\G_{\p} = (U,E,u_{0},F,\beta,\traj_{\safe})$ where
\begin{itemize}
\item $U = U_1$, $u_0 = u_0^1$, $E=E_1$ and $F=F_1$.
\item $\beta(u) = \beta_1(u)\cap S_b$ for all $u\in U\setminus\{u_0\}$ where $S_b = \{s\mid s\models b\}$, and $\beta(u_0) = S$.
\item $\traj_{\safe}^{e} = \traj_{\safe,1}^e\cap\traj_b$ for all $e\in E$ where
$$\traj_{b} = \{\zeta\in\traj_f\mid \forall i\;.\; s_i\models b\}.$$
\item $\traj_{\term}^{u} = \traj_{\term,1}^u\cap\traj_b$ for all $u\in F$.
\end{itemize}

%Then, the monitor for $\p$ is $\A_{\p} = (Q, \Delta, Q_F, q_{0})$ where:
% \begin{itemize}
%     \item $Q = Q_1$, $Q_F = Q_F^1$ and $q_0 = q_0^1$.
%     \item $\Delta = \{(q, b_1\land b,q')\mid (q,b_1,q')\in \Delta_1\}$.
% \end{itemize}

\textbf{Sequencing} ($\p = \p_1;\p_2$). Let the abstract graph for $\p_i$ be
% $\A_{\p_1}$ and $\A_{\p_2}$ be given by
% $\A_{\p_i} = (Q_i, \Delta_i, Q_F^i, q_{0}^i)$
% for $i\in\{1,2\}$. Similarly, let
$\G_{\p_{i}} = (U_i,E_i, u_{0}^i,F_i,\beta_i,\traj_{\safe,i})$ and the terminal trajectories be $\traj_{\term,i}$ for $i\in\{1,2\}$.
% Then the monitor for $\p$ is $\A_{\p} = (Q, \Delta, Q_F, q_{0})$ where:
% \begin{itemize}
%     \item $Q = Q_1 \sqcup Q_2$ where $\sqcup$ denotes disjoint union.
%     \item $\Delta = \Delta_1\sqcup\Delta_2\sqcup{\Delta_{1\to 2}}$ where $\Delta_1$ and $\Delta_2$ are the same as in $\A_{\p_1}$ and $\A_{\p_2}$. Let $(q^2_0,b_1,q^2_1),\ldots,(q^2_0,b_k,q^2_k)$ be all the outgoing transitions from $q_0^2$ in $\A_{\p_2}$. Then $\Delta_{1\to 2} = \{(q^1_F,b_j,q_j^2)\mid q^1_F\in Q^1_F, j\in\{1,\ldots,k\}\}$.
%     \item $Q_F = Q^2_F$ and $q_{0} = q_{0}^1$.
% \end{itemize}
The abstract graph $\G_{\p} = (U, E, u_0, F, \beta,\traj_{\safe})$ is constructed as follows.
\begin{itemize}
\item $U = U_1\sqcup U_2 \setminus \{u_{0}^2\}$.
\item $E = E_1\sqcup E_2'\sqcup E_{1\to 2}$ where
$$E_2' = \{u\to u'\in E_2\mid u\neq u_0^2\}\quad\text{and}$$
$$E_{1\to 2} = \{u^1\to u^2\mid u^1\in F_1\ \&\ u_0^2\to u^2\in E_2\}.$$
\item $u_0 = u_0^1$ and $F = F_2$.
\item $\beta(u) = \beta_i(u)$ for all $u\in U_i$ and $i\in\{1,2\}$.
\item The safe trajectories are given by
\begin{itemize}
\item $\traj_{\safe}^e = \traj_{\safe,1}^e$ for all $e \in E_1$,
\item $\traj_{\safe}^e = \traj_{\safe,2}^e$ for all $e\in E_{2}'$ and,
\item $\traj_{\safe}^{u^1\to u^2} = \traj_{\term,1}^{u^1}\circ\traj_{\safe,2}^{u_0^2\to u^2}$ for all $u^1\to u^2\in E_{1\to 2}$.
\end{itemize}
\item $\traj_{\term}^u = \traj_{\term,2}^u$ for all $u \in F$.
\end{itemize}

\textbf{Choice} ($\p = \choice{\p_1}{\p_2}$). Let the abstract graph for $\p_i$ be $\G_{\p_{i}} = (U_i,E_i, u_{0}^i,F_i,\beta_i,\traj_{\safe,i})$ and the terminal trajectories be $\traj_{\term,i}$ for $i\in\{1,2\}$.
% Let $(q^i_0,b^i_1,q^i_1),\ldots,(q^i_0,b^i_{k_i},q^i_{k_i})$ be all the outgoing transitions from $q_0^i$ in $\A_{\p_i}$ where $i\in\{1,2\}$. Then, $\A_{\p} = (Q,\Delta, Q_F, q_{0})$ where:
% \begin{itemize}
%     \item $Q = Q_1\sqcup Q_2\sqcup\{q_0\}$ where $q_0$ is a new monitor state which is also the initial state.
%     \item $\Delta = \Delta_1\sqcup\Delta_2\sqcup\Delta_0$ where $\Delta_0 = \{(q_0, b^i_j, q_j^i)\mid i\in\{1,2\}, j\in\{1,\ldots,k_i\}\}$.
%     \item $Q_F = Q_F^1\cup Q_F^2$
% \end{itemize}
The abstract graph for $\p$ is $\G_{\p} = (U, E, u_0, F,\beta, \traj_{\safe})$ where:
\begin{itemize}
\item $U = \Big(U_1\setminus\{u_0^1\}\Big)\sqcup \Big(U_2\setminus\{u_0^2\}\Big)\sqcup\{u_0\}$.
\item $E = E_1'\sqcup E_2'\sqcup E_{0}$ where $$E_i' = \{u\to u'\in E_i\mid u\neq u_0^i\}\quad \text{and}$$ $$E_{0} = \{u_0\to u^i\mid i\in\{1,2\}\ \&\ u_0^i\to u^i\in E_i\}.$$
\item $F = F_1\sqcup F_2$.
\item $\beta(u) = \beta_i(u)$ for all $u\in U_i$, $i\in\{1,2\}$ and $\beta(u_0) = S$.
\item The safe trajectories are given by
\begin{itemize}
\item $\traj_{\safe}^e = \traj_{\safe,i}^e$ for all $e\in E_i'$ and $i\in\{1,2\}$,
\item $\traj_{\safe}^{u_0\to u^i} = \traj_{\safe,i}^{u_0^i\to u^i}$ for all $u_0\to u^i\in E_0$ with $u^i \in U_i$.
\end{itemize}
\item $\traj_{\term}^u = \traj_{\term,i}^u$ for all $u\in F_i$ and $i\in\{1,2\}$.
\end{itemize}

The constructed pair $(\G_{\p},\traj_{\term,\p})$ has the following important properties.

\begin{lemma}\label{lem:finite_eq}
For any \toolname specification $\p$, the following hold.
\begin{itemize}
    \item For any finite trajectory $\zeta\in\traj_f$, $\zeta\models\p$ if and only if $\zeta\models(\G_{\p},\traj_{\term,\p})$.
    \item For any final vertex $u$ of $\G_\p$ and any state $s\in\beta(u)$, the length-1 trajectory $\zeta=s$ is contained in $\traj_{\term,\p}^u$.
\end{itemize}
\end{lemma}
\begin{proof}
Follows from the above construction by structural induction on $\p$.
\end{proof}

\begin{proof}[Proof of Theorem~\ref{thm:reduction}]
Let $\zeta = s_0\xrightarrow{a_0}s_1\xrightarrow{a_1}\cdots$ be an infinite trajectory. First we show that $\zeta\models\p$ if and only if $\zeta\models\G_\p$.

($\implies$) Suppose $\zeta\models\p$. Then, there is a $t\geq 0$ such that $\zeta_{0:t}\models\p$. From Lemma~\ref{lem:finite_eq}, we get that $\zeta_{0:t}\models(\G_{\p},\traj_{\term,\p})$ which implies that $\zeta\models\G_\p$.

($\impliedby$) Suppose $\zeta\models\G_\p$. Then, let $0=i_0\leq i_1 < \cdots < i_k$ be a sequence of indices realizing a path $u_0\to\cdots\to u_k$ to a final vertex $u_k$ in $\G_\p$. Since $s_{i_k}\in\beta(u_k)$, from Lemma~\ref{lem:finite_eq} we have $\zeta_{i_k:i_k}\in\traj_{\term,\p}^{u_k}$ and hence $\zeta_{0:i_k}\models(\G_{\p},\traj_{\term,\p})$. From Lemma~\ref{lem:finite_eq} we conclude that $\zeta_{0:i_k}\models\p$ and therefore $\zeta\models\p$.

Next, it follows by a straightforward induction on $\p$ that the number of vertices in $\G_\p$ is at most $|\p|+1$ where $|\p|$ is the number of operators (\code{achieve}, \code{ensuring}, ;, \code{or}) in $\p$.
\end{proof}
\section{Shaped Rewards for Learning Policies}
\label{Ap:shaped_rewards}

To improve learning, we use shaped rewards for learning each edge policy $\pi_e$. To enable reward shaping, we assume that the atomic predicates additionally have a \emph{quantitative semantics}---i.e., each atomic predicate $p\in\P_0$ is associated with a function $\semantics{p}_q:S\to\mathbb{R}$. To ensure compatibility with the Boolean semantics, we assume that
\begin{align}
\semantics{p}(s)=\big(\semantics{p}_q(s)>0\big).
\label{eqn:quantitativetoboolean}
\end{align}
For example, given a state $s\in S$, the atomic predicate
\begin{align*}
\semantics{\reach s}_q(s') ~=~ 1 - \|s' - s\|
\end{align*}
indicates whether the system is in a state near $s$ w.r.t. some norm $\|\cdot\|$. In addition, we can extend the quantitative semantics to predicates $b\in\mathcal{P}$ by recursively defining $\semantics{b_1\wedge b_2}_q(s)=\min\{\semantics{b_1}_q(s),\semantics{b_2}_q(s)\}$ and $\semantics{b_1\vee b_2}_q(s)=\max\{\semantics{b_1}_q(s),\semantics{b_2}_q(s)\}$. These definitions are a standard extension of Boolean logic to real values. In particular, they preserve (\ref{eqn:quantitativetoboolean})---i.e., $b\models s$ if and only if $\semantics{b}_q(s)>0$.

In addition to quantitative semantics, we make use of the following property to define shaped rewards. 

\begin{lemma}\label{lem:g_property}
The abstract graph $\G_\p = (U, E,u_0, F, \beta,\traj_{\safe})$ of a specification $\p$ satisfies the following:
\begin{itemize}
\item For every non-initial vertex $u\in U\setminus\{u_0\}$, there is a predicate $b\in\P$ such that $\beta(u) = S_b = \{s\mid s\models b\}$.
\item For every $e\in E$, either $\traj_{\safe}^e = \traj_{b} = \{\zeta\in\traj\mid \forall i\;.\; s_i\models b\}$ for some $b\in\P$ or $\traj_{\safe}^e = \traj_{b_1}\circ\traj_{b_2}$ for some $b_1,b_2\in\P$.
\end{itemize}
\end{lemma}
\begin{proof}[Proof sketch]
We prove a stronger property that, in addition to the above, requires that for any $e = u_0\to u\in E$, $\traj_{\safe}^e = \traj_{b}$ for some $b\in\P$ and for any final vertex $u$, $\traj_{\term,\p}^u = \traj_{b}$ for some $b\in\P$. This stronger property follows from a straightforward induction on $\p$.
\end{proof}

Next, we describe the shaped rewards we use to learn an edge $e=u\to u'$ in $\G_\p$, which have the form
\begin{align*}
R_{\text{step}}(s,a,s') = R_{\text{reach}}(s,a,s') + R_{\text{safe}}(s,a,s').
\end{align*}
Intuitively, the first term encodes a reward for reaching $\beta(u')$, and the second term encodes a reward for maintaining safety.
By Lemma~\ref{lem:g_property}, $\beta(u') = S_b$ for some $b\in\P$. Then, we define
\begin{align*}
R_{\text{reach}}(s,a,s') = \semantics{b}_q(s').
\end{align*}
The safety reward is defined by
\begin{align*}
R_{\text{safe}}(s,a,s') 
&=\begin{cases}
\min\{0, \semantics{b}_q(s')\} & \text{if }\traj_{\safe}^e=\traj_{b} \\
\min\{0, \semantics{b\lor b'}_q(s')\} & \text{if }\traj_\safe^e=\traj_{b}\circ\traj_{b'}\wedge\psi_b\\
\min\{0, \semantics{b'}_q(s')\} & \text{if }\traj_\safe^e=\traj_{b}\circ\traj_{b'}\wedge\neg\psi_b.
\end{cases}
\end{align*}
Here, $\psi_b$ is internal state keeping track of whether $b$ has held so far---i.e., $\psi_b\gets\psi_b\wedge\semantics{b}(s)$ at state $s$. Intuitively, the first case is the simpler case, which checks if every state in the trajectory satisfies $b$, and the latter two cases handle a sequence where $b$ should hold for the first part of the trajectory, and $b'$ should hold for the remainder.

\section{Proof of Theorem~\ref{thm:main}}\label{sec:proofs}

\begin{proof} Let the abstract graph be $\G = (U,E,u_0,F,\beta,\traj_\safe)$. Let us first define what it means for a rollout to achieve a path in $\G$.

\begin{definition} We say that an infinite trajectory $\zeta$ achieves the path $\rho$ (denoted $\zeta\models\rho$) if $\zeta\models\G_{\rho}$ where $\G_{\rho} = (U_\rho, E_{\rho}, u_0, \{u_k\},\beta\downarrow\rho, \traj_\safe\downarrow_\rho)$ with $U_{\rho} = \{u_j ~\mid~ 0\leq j\leq k\}$, $E_{\rho} = \{u_j\rightarrow u_{j+1} ~\mid~ 0\leq j < k\}$ and $\beta\downarrow\rho$ and $\traj_{\safe}\downarrow_\rho$ are $\beta$ and $\traj_\safe$ restricted to the vertices and the edges of $\G_{\rho}$, respectively.
\end{definition}

From the definition it is clear that for any infinite trajectory $\zeta$, if $\zeta\models\rho$ then $\zeta\models\G$ and therefore
\begin{align}\label{eqn:bound1}
    \Pr_{\zeta\sim\D_{\pi_{\rho}}}[\zeta\models\G]\geq\Pr_{\zeta\sim\D_{\pi_{\rho}}}[\zeta\models\rho].
\end{align}
Let us now define a slightly stronger notion of achieving an edge.
\begin{definition}
An infinite trajectory $\zeta=s_0\rightarrow s_1\rightarrow\cdots$ is said to greedily achieve the path $\rho$ (denoted $\zeta\models_g\rho$) if there is a sequence of indices $0=i_0\leq i_1<\cdots<i_k$ such that for all $j<k$,
\begin{itemize}
\item $\zeta_{i_j:\infty}\models e_j=u_j\rightarrow u_{j+1}$ and,
\item $i_{j+1} = i(\zeta_{i_j:\infty},e_j)$,
\end{itemize}
where $\zeta_{i_j:\infty}=s_{i_j}\rightarrow s_{i_j + 1}\rightarrow\cdots$.
\end{definition}
That is, $\zeta\models_g\rho$ if a partition of $\zeta$ realizing $\rho$ can be be constructed greedily by picking $i_{j+1}$ to be the smallest index $i \geq i_j$ (strictly bigger if $j>0$) such that $s_i\in \beta(u_{j+1})$ and $\zeta_{i_j:i}\in \traj_{\safe}^{e_j}$. Since $\zeta\models_g\rho$ implies $\zeta\models\rho$, we have
\begin{align}\label{eqn:bound2}
    \Pr_{\zeta\sim\D_{\pi_{\rho}}}[\zeta\models\rho]\geq\Pr_{\zeta\sim\D_{\pi_{\rho}}}[\zeta\models_g\rho].
\end{align}
Let $\rho_{j:k}$ denote the $j$-th suffix of $\rho$. We can decompose the probability $\Pr_{\zeta\sim\D_{\pi_{\rho}}}[\zeta\models_g\rho]$ as follows.
\begin{align*}
    \Pr_{\zeta\sim\D_{\pi_{\rho}}}[\zeta\models_g\rho] &=
    \Pr_{\zeta\sim\D_{\pi_{\rho}}}[\zeta\models e_0\ \wedge\ \zeta_{i(\zeta,e_0):\infty}\models_g \rho_{1:k}]\\
    &=\Pr_{\zeta\sim\D_{\pi_{e_0}}}[\zeta\models e_0]\cdot\Pr_{\zeta\sim\D_{\pi_{\rho}}}[\zeta_{i(\zeta,e_0):\infty}\models_g \rho_{1:k}\mid \zeta\models e_0]\\
    &=P(e_0;\pi_{e_0},\eta_0)\cdot \Pr_{s_0\sim\eta_{\rho_{0:1}},\zeta\sim\D_{\pi_{\rho_{1:k}},s_0}}[\zeta\models_g \rho_{1:k}]
\end{align*}
where the last equality followed from the definition of $\eta_{\rho_{0:1}}$ and the Markov property of $\M$. Applying the above decomposition recursively, we get
\begin{align*}
    \Pr_{\zeta\sim\D_{\pi_{\rho}}}[\zeta\models_g\rho] &= \prod_{j=0}^{k-1} P(e_j;\pi_{e_j},\eta_{\rho_{0:j}})\\
    &= \exp(\log(\prod_{j=0}^{k-1} P(e_j;\pi_{e_j},\eta_{\rho_{0:j}})))\\
    &= \exp(-(-\sum_{j=0}^{k-1}\log P(e_j;\pi_{e_j},\eta_{\rho_{0:j}})))\\
    &= \exp(-c(\rho)).
\end{align*}
Therefore, from Equations~\ref{eqn:bound1} and \ref{eqn:bound2}, we get the required bound.
\end{proof}

\section{Experimental Methodology}
\label{ap:methodology}

Our tool learns the low-level NN policies for edges using an off-the-shelf RL algorithm. For the Rooms and Fetch environments, we learn policies using ARS~\cite{mania2018simple} and TD3~\citep{fujimoto2018addressing} with shaped rewards, respectively.  

For each specification on an environment, we first construct its abstract graph. In \dirl, each edge policy $\pi_e$ is trained using $k$ episodes of interactions with the environment.
For the purpose of generating a learning curve, we run \dirl for each specification with several values of $k$. For each $k$ value, we plot the sum total of the samples taken to train all edge policies against the probability with which the computed policy reaches a final subgoal region. 

For a fair comparison with the baselines, if each episode for learning an edge policy in \dirl is run for $m$ steps, we run the episodes of the baselines for $m\cdot d + c$ steps, where $d$ is the maximum path length to reach a final vertex in the abstract graph of the specification and $c>0$ is a buffer. Intuitively, this approach ensures that all tools get a similar number of steps in each episode to learn the specification.

\begin{figure*}
\centering
\begin{subfigure}{0.3\textwidth}
\centering
\includegraphics[width=0.8\linewidth]{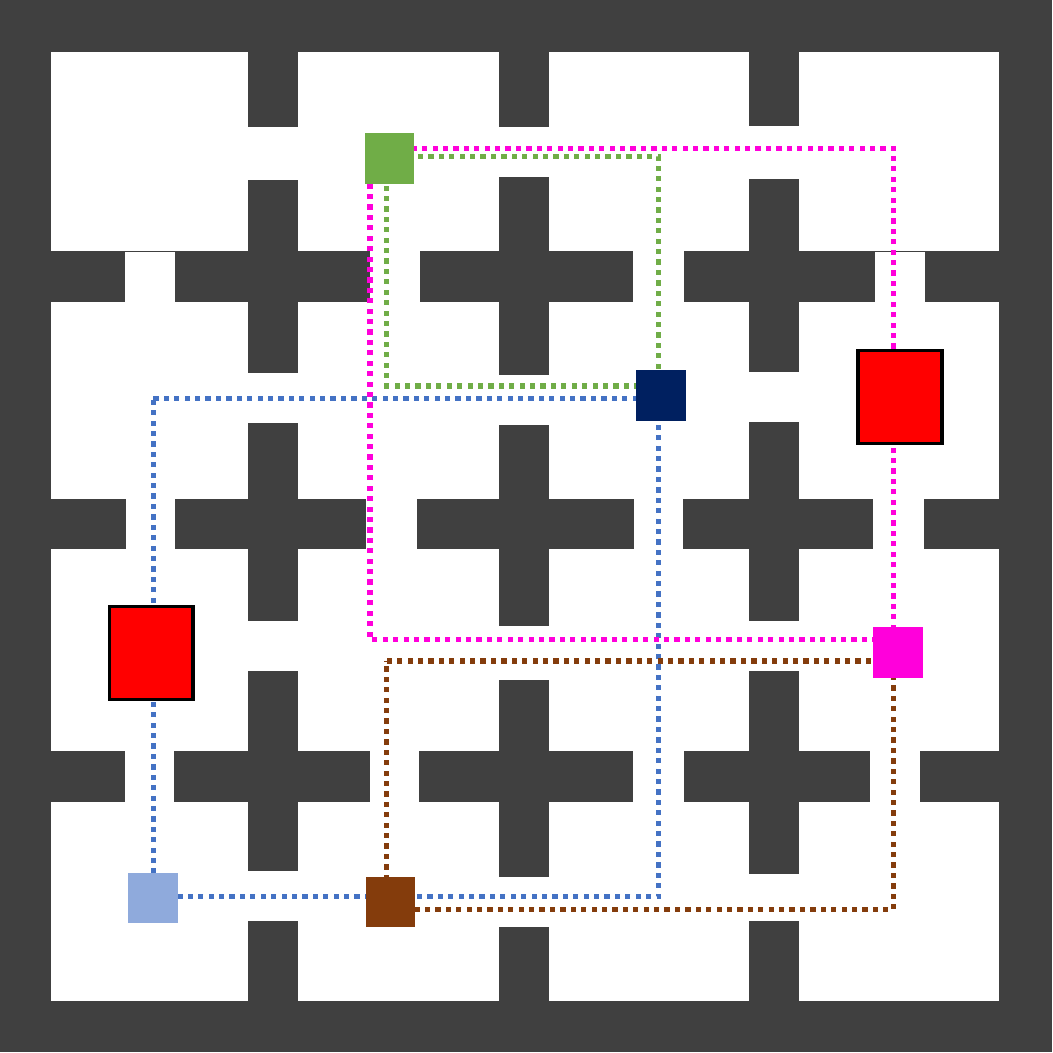}
\caption{16-Rooms (All doors open)}
\label{fig:16roomsopendoors}
\end{subfigure}
%\hfill
\begin{subfigure}{0.3\textwidth}
\centering
\includegraphics[width=0.8\linewidth]{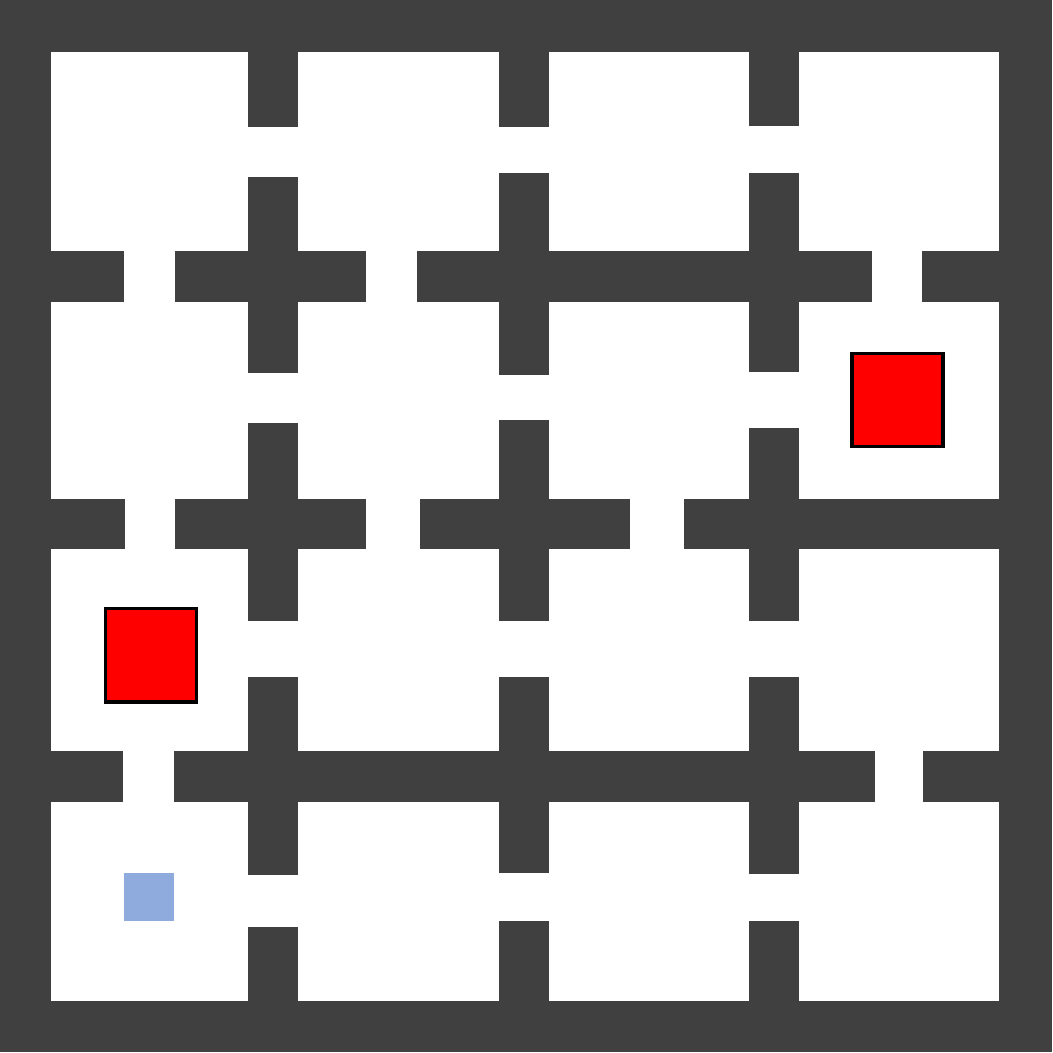}
\caption{16-Rooms (Some doors open)}
\label{fig:16rooms}
\end{subfigure}
\caption{16-Rooms Environments. Blue square indicates the initial room. Red squares represent obstacles. (a) illustrates the segments in the specifications. }
\label{Fig:RoomsEnv}
\end{figure*}

\section{Case Study: Rooms Environment}
\label{Ap:RoomsCaseStudy}

We consider environments with several interconnected rooms. The rooms are separated by thick walls and are connected through bi-directional doors. 
%They have states $(x,y)\in\R^2$ encoding 2D position, actions $(v,\theta)\in\R^2$ encoding speed and direction, and transitions $s_0 = s + (v \cos{\theta}, v \sin\theta)$.

The environments are a 9-Rooms environment,   (\autoref{fig:Motivate}), a 16-Rooms environment with all doors open (\autoref{fig:16roomsopendoors}), and 16-Rooms environment with some doors open (\autoref{fig:16rooms}).
The red blocks indicate obstacles. A robot can pass through those rooms by moving around the red blocks. 
The robot is initially placed randomly in the center of the room with the blue box (bottom-left corner). 

Rooms are identified by the tuple $(r,c)$ denoting the room in the $r$-th row and $c$-th column. We use the convention that the bottom-left corner is room (0,0). Predicate $\reach{(r,c)}$ is interpreted as reaching the center of the $(r,c)$-th room and predicate $\avoid{(r,c)}$ is interpreted as avoiding the center of the $(r,c)$-th room. For clarity, we omit the word \code{achieve} from specifications of the form $\eventually{b}$ denoting such a specification using just the predicate $b$.
%Initially, the RL-agent randomly placed at the center of room $(0,0)$. 

\subsection{9-Rooms Environment}

\textbf{Specifications.}
\begin{enumerate}
\item {$\p_1 :=$}
$\reach(2,0)$; $\reach(0,0)$

Go to the top-left corner and then return to the bottom-left corner (initial room); red blocks not considered obstacles.

This specification is difficult for standard RL algorithms that do not store whether the first sub-task has been achieved. In these cases, a stateless policy will not be able to determine whether to move upwards or downwards. In contrast, \dirl (as well as \toolname and RM based approaches) augment the state space to automatically keep track of which sub-tasks have been achieved so far.

\item{$\p_2 :=$}
$\choice{\reach(2,0)}\reach(0,2)$

Either go to the top-left corner or to the bottom-right corner (obstacles are not considered).

\item {$\p_3 :=$}
$\p_2; \reach(2,2)$

After completing $\p_2$, go to the top-right corner (obstacles not considered).

This specification combines two choices of similar difficulty yet only one is favorable to fulfilling the specification since the direct path to the top-right corner from the bottom-right one is obstructed by walls.

\item {$\p_4 :=$}
$\reach(2,0)\always$ $\avoid(1,0)$

Reach the top-left (while considering the obstacles).

\item {$\p_5 :=$}
$\choice{\p_4}\reach(0,2); $ $\reach(2,2)$

Either go to the top-left corner or bottom-right corner enroute to the top-right corner (while considering the obstacles).

This specification is similar to $\p_3$ except that the choices are of unequal difficulty due to the placement of the red obstacle. In this case, the non-greedy choice is favorable for completing the task. 

\end{enumerate}

\begin{figure*}[t]
\centering

\begin{subfigure}{0.3\textwidth}
\centering
\includegraphics[width=\textwidth]{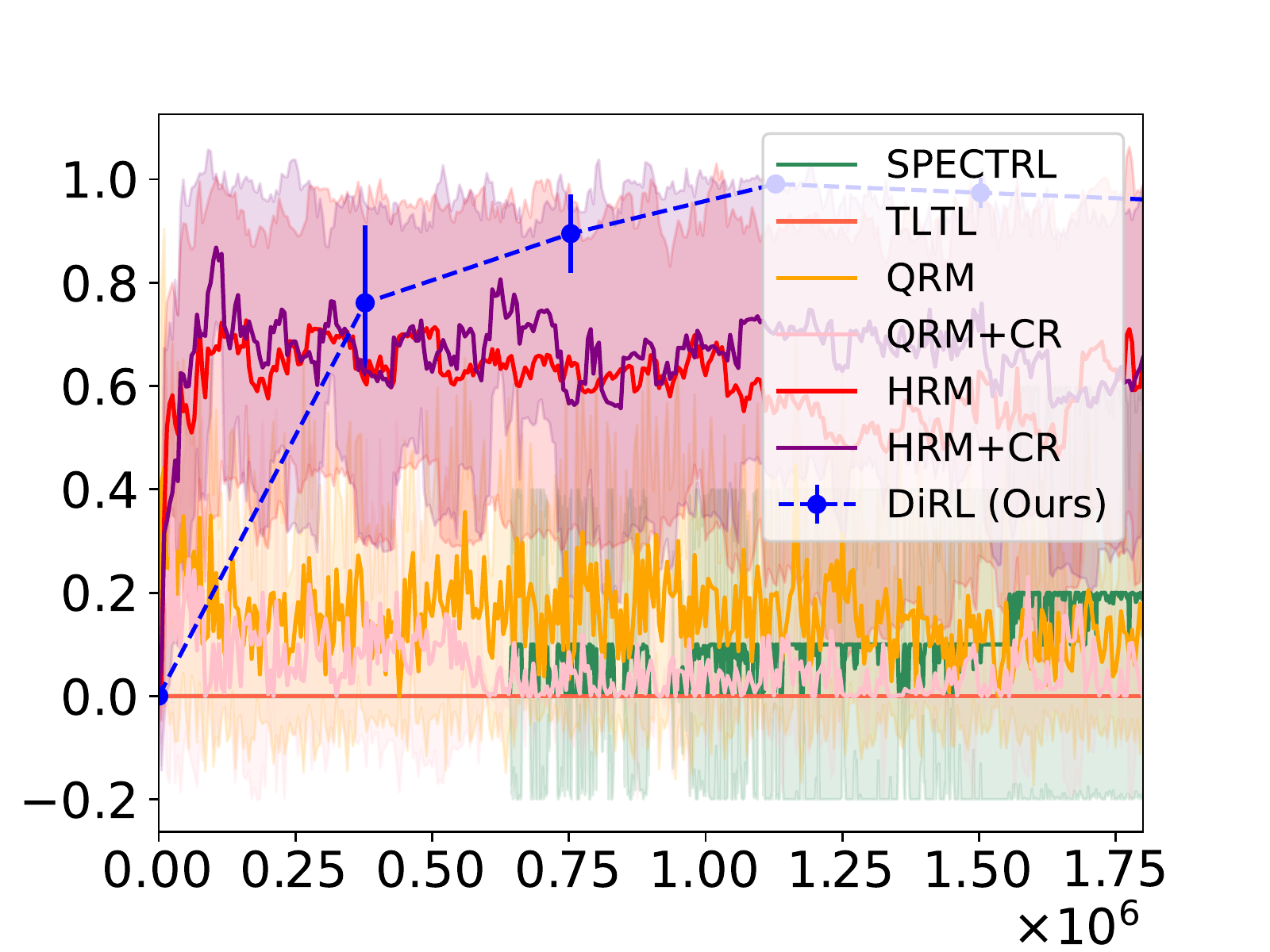}
\caption{Specification $\p_1$}
\label{fig:9rooms3}
\end{subfigure}
\hfill
\begin{subfigure}{0.3\textwidth}
\centering
\includegraphics[width=\textwidth]{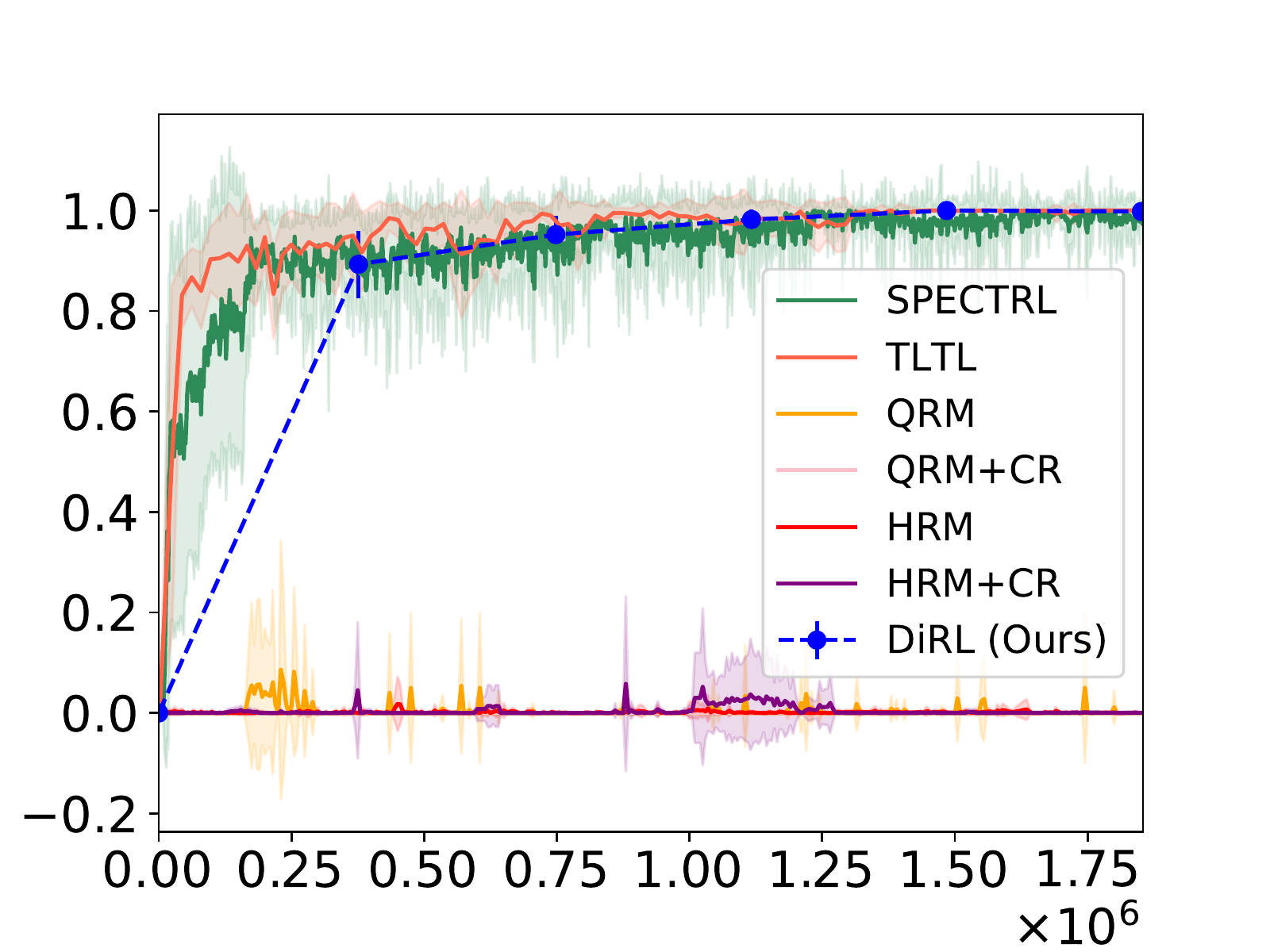}
\caption{Specification $\p_2$}
\label{fig:9rooms4}
\end{subfigure}
\hfill
\begin{subfigure}{0.3\textwidth}
\centering
\includegraphics[width=\textwidth]{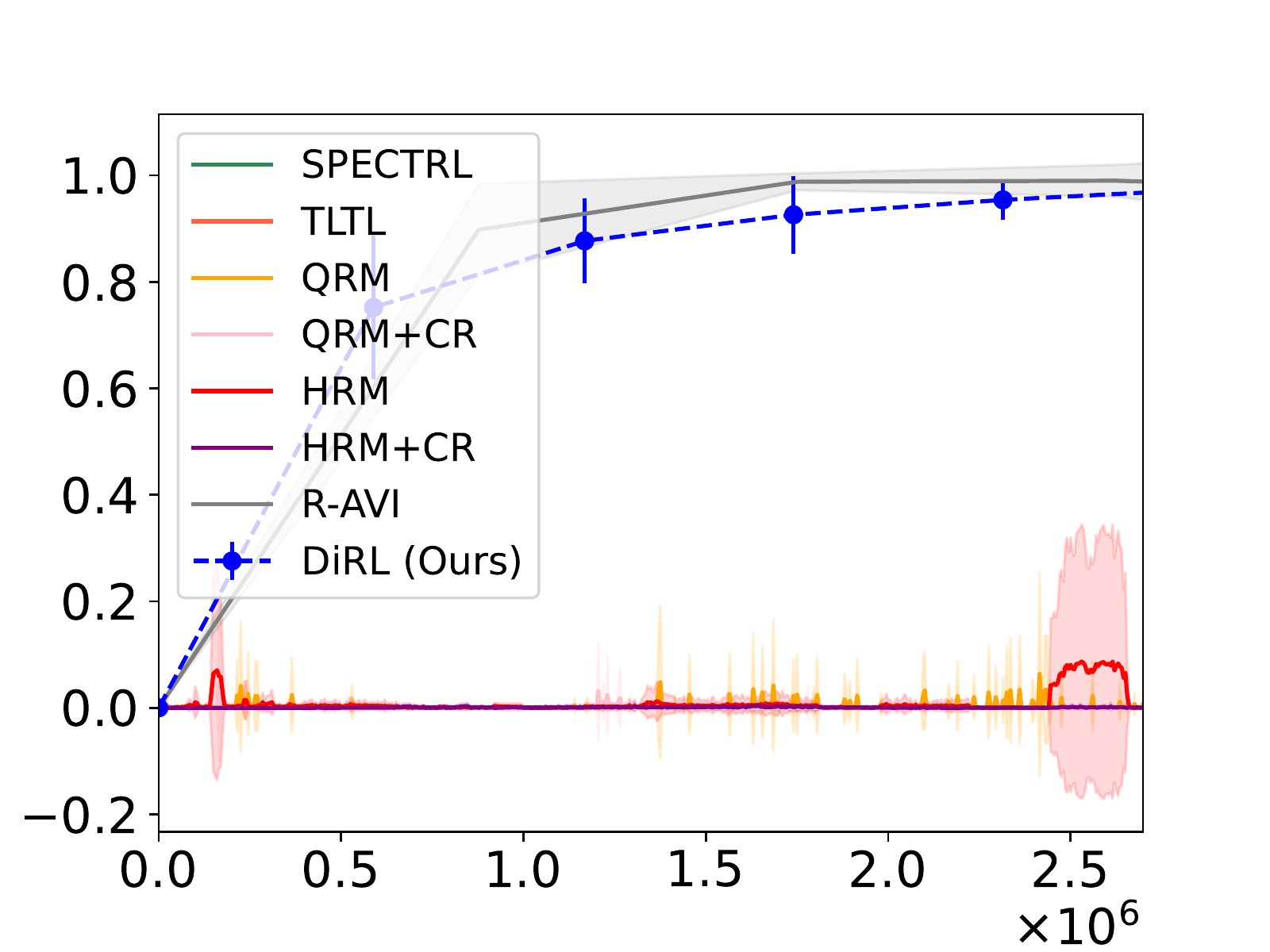}
\caption{Specification $\p_3$}
\label{fig:9rooms5}
\end{subfigure}
\hfill         
\begin{subfigure}{0.3\textwidth}
\centering
\includegraphics[width=\textwidth]{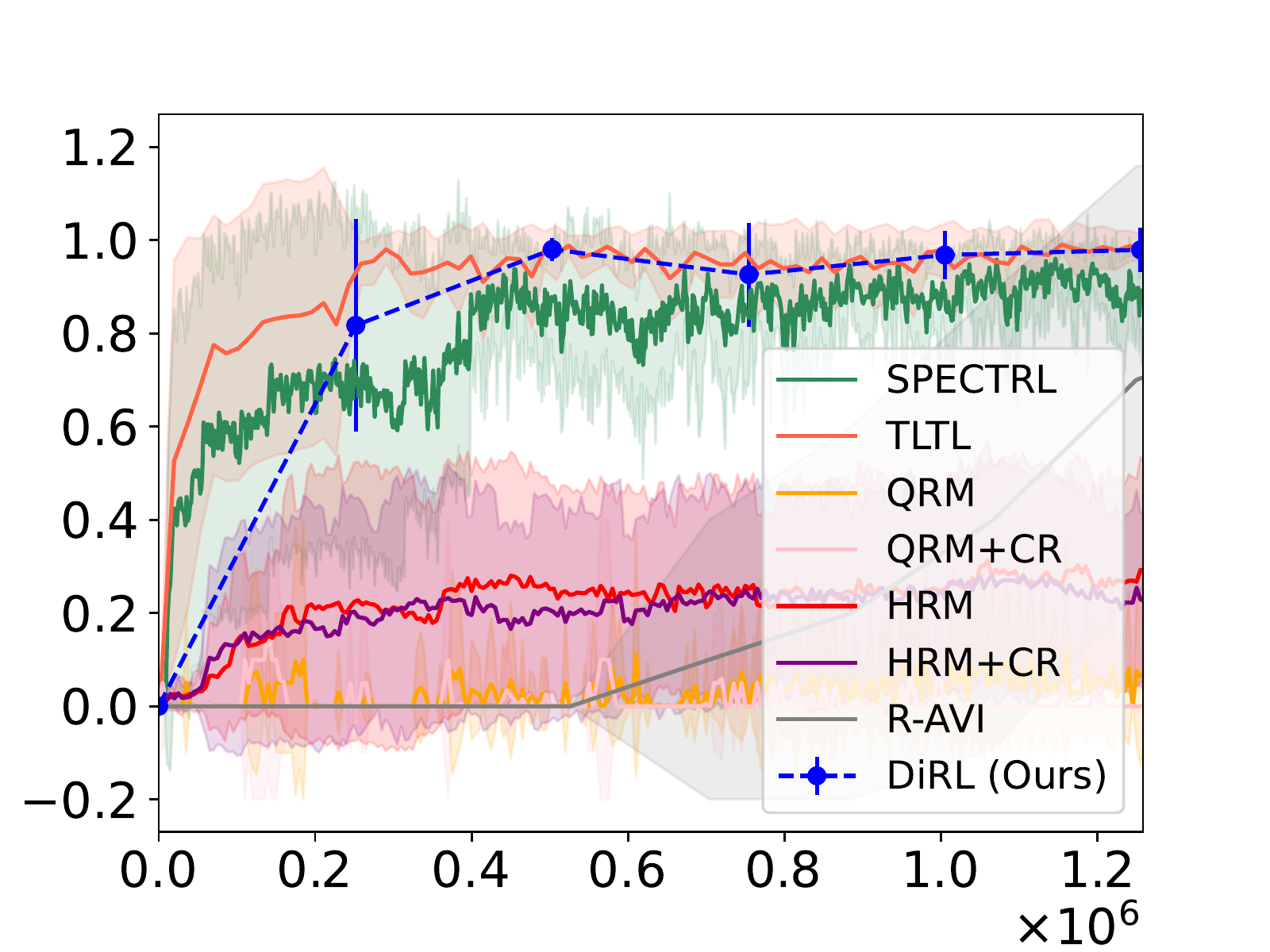}
\caption{Specification $\p_4$}
\label{fig:9rooms6}
\end{subfigure}
%\hfill          
\begin{subfigure}{0.3\textwidth}
\centering
\includegraphics[width=\textwidth]{Plots/rooms9_spec7_0.pdf}
\caption{Specification $\p_5$}
\label{fig:9rooms7}
\end{subfigure}
\caption{Learning curves for 9-Rooms environment with different specifications. $x$-axis denotes the number of samples (steps) and $y$-axis denotes the estimated probability of success. Results are averaged over 10 runs with error bars indicating $\pm$ standard deviation.}
\label{Fig:9RoomsLC}
\end{figure*}

\textbf{Hyperparameters.}
The edge policies are learned using ARS \citep{mania2018simple} (version V2-t) with neural network policies and the following hyperparameters.
\begin{itemize}
\item Step-size $\alpha = 0.3$.
\item Standard deviation of exploration noise $\nu = 0.05$.
\item Number of directions sampled per iteration is $30$.
\item Number of top performing directions to use $b = 15$.
\end{itemize}
To plot the learning curve, we use values of
\begin{align*}
k \in \{3000, 6000, 12000, 18000, 24000, 30000\}
\end{align*}
where each episode consists of $m=20$ steps.

\textbf{Results.}
The learning curves for these specifications are shown in \autoref{Fig:9RoomsLC}.
While most tools perform reasonably well on specifications $\p_2$  (\autoref{fig:9rooms4}) and $\p_4$ (\autoref{fig:9rooms6}), the baselines are unable to learn to satisfy $\p_3$ (\autoref{fig:9rooms5}) and $\p_5$ (\autoref{fig:9rooms7}) except for \textsc{R-avi} which learns to satisfy $\p_3$ as well. 

%This demonstrates the merit of our Dijkstra-based algorithm as the interleaving of high-level and low-level policies encourages the RL-agent to explore future edges (learn polices for future edges) in the abstract graph before committing to a policy for each edge. 

\subsection{16-Rooms Environment}

\textbf{Specifications.}
We describe the five specifications used for the 16-rooms environment, which are designed to increase in difficulty. First, we define a \emph{segment} as the following specification: Given the current location of the agent, the goal is to reach a room diagonally opposite to it by visiting at least one of the rooms at the remaining two corners of the rectangle formed by the current room and the goal room---e.g., in the 9-Rooms environment, to visit $S_3$ from the initial room, the agent must visit either $S_1$ or $S_2$ first.

Then, we design specifications of varying sizes by sequencing several segments one after the other. In the first segment, the agent's current location is the initial room. In subsequent segments, the current location is the goal room of the previous segment. In addition, the agent must always avoid the obstacles in the environment. We create five such specifications, one half-segment and specifications up to four segments ($\p_1$ to $\p_5$), as illustrated in \autoref{fig:16roomsopendoors} and described below:

\begin{enumerate}
\item $\p_1$ corresponds to the {\em half-segment} enroute (2,2) from (0,0). Thus $\p_1$ is a choice between (0,2) and (2,0).
\item $\p_2$ is the first segment that goes from (0,0) to (2,2)
\item $\p_3$ augments $\p_2$ with a second segment to (3,1).
\item $\p_4$ augments $\p_3$ with a segment to (1,3)
\item $\p_5$ augments $\p_4$ with a segment to (0,1)
\end{enumerate}

We denote by $|\G_{\p}|$ the number of edges in the abstract graph corresponding to the specification $\p$.

\begin{figure*}[t]
\centering
\begin{subfigure}{0.3\textwidth}
\centering
\includegraphics[width=\textwidth]{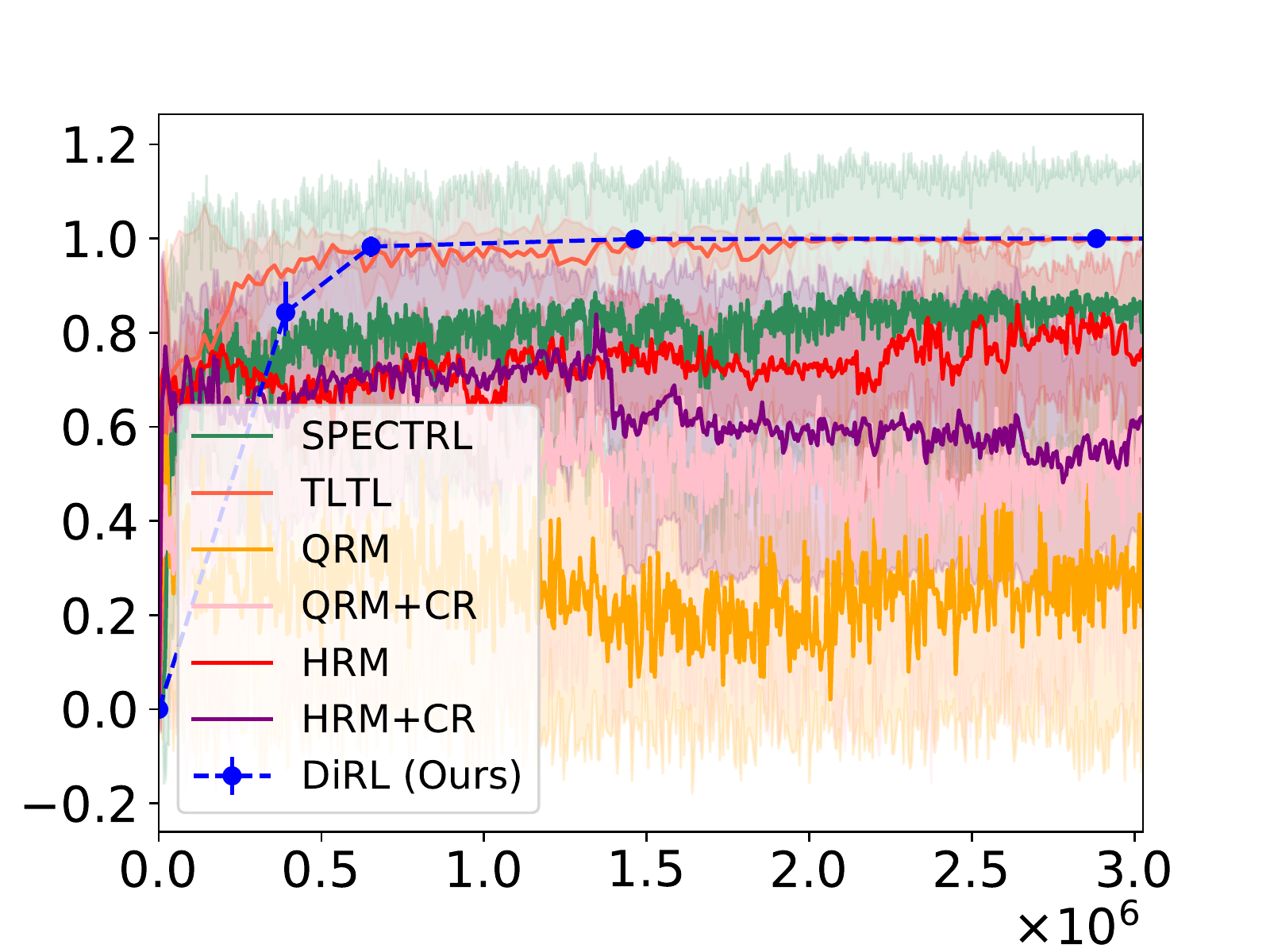}
\caption{Half-segment spec. $\p_1$, $|\G_{\p_1}| = 2$.}
\label{fig:16_4rooms9}
\end{subfigure}
\hfill
\begin{subfigure}{0.3\textwidth}
\centering
\includegraphics[width=\textwidth]{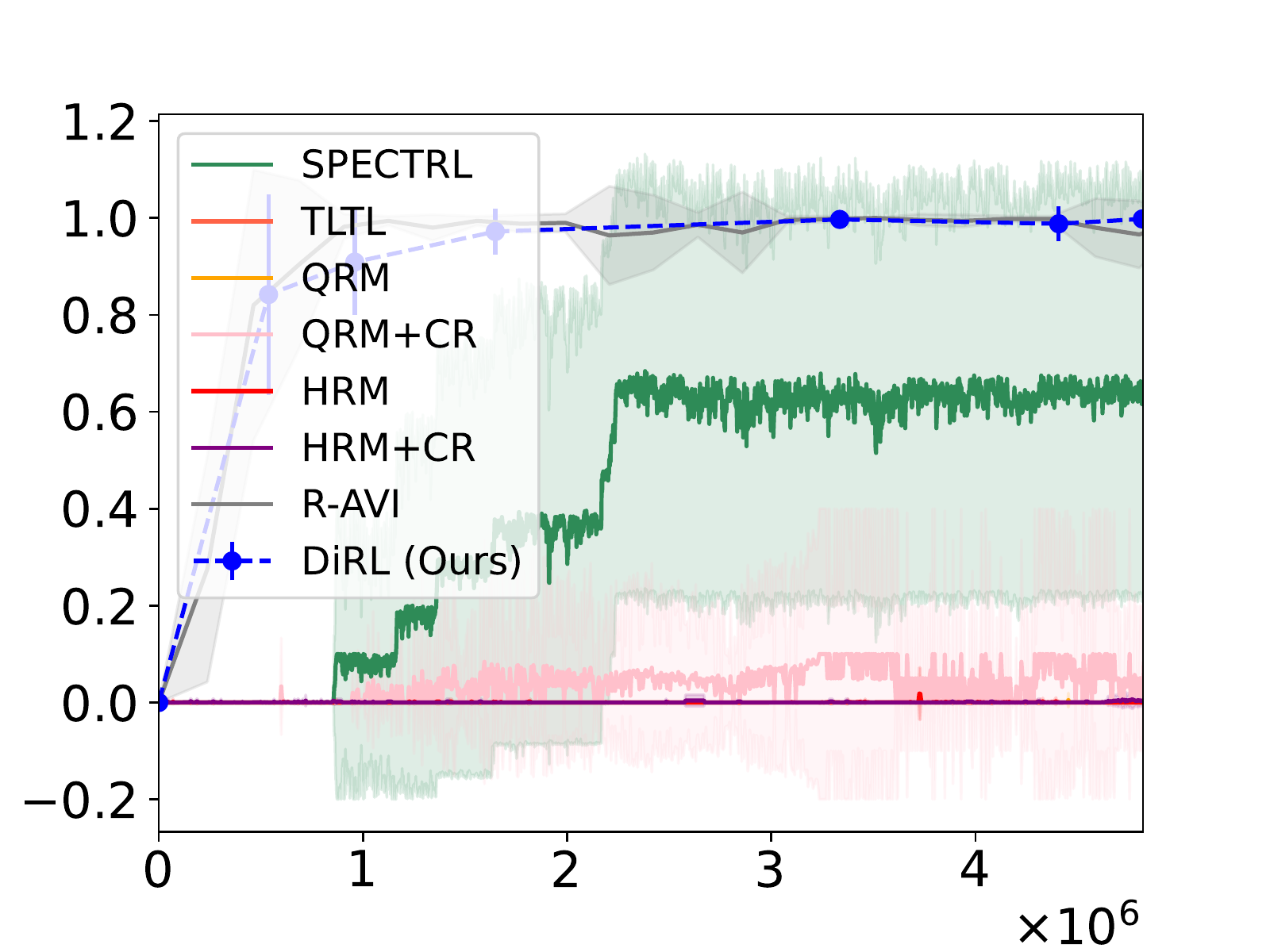}
\caption{1-segment spec. $\p_2$, $|\G_{\p_2}| = 4$.}
\label{fig:16_4rooms10}
\end{subfigure}
\hfill
\begin{subfigure}{0.3\textwidth}
\centering
\includegraphics[width=\textwidth]{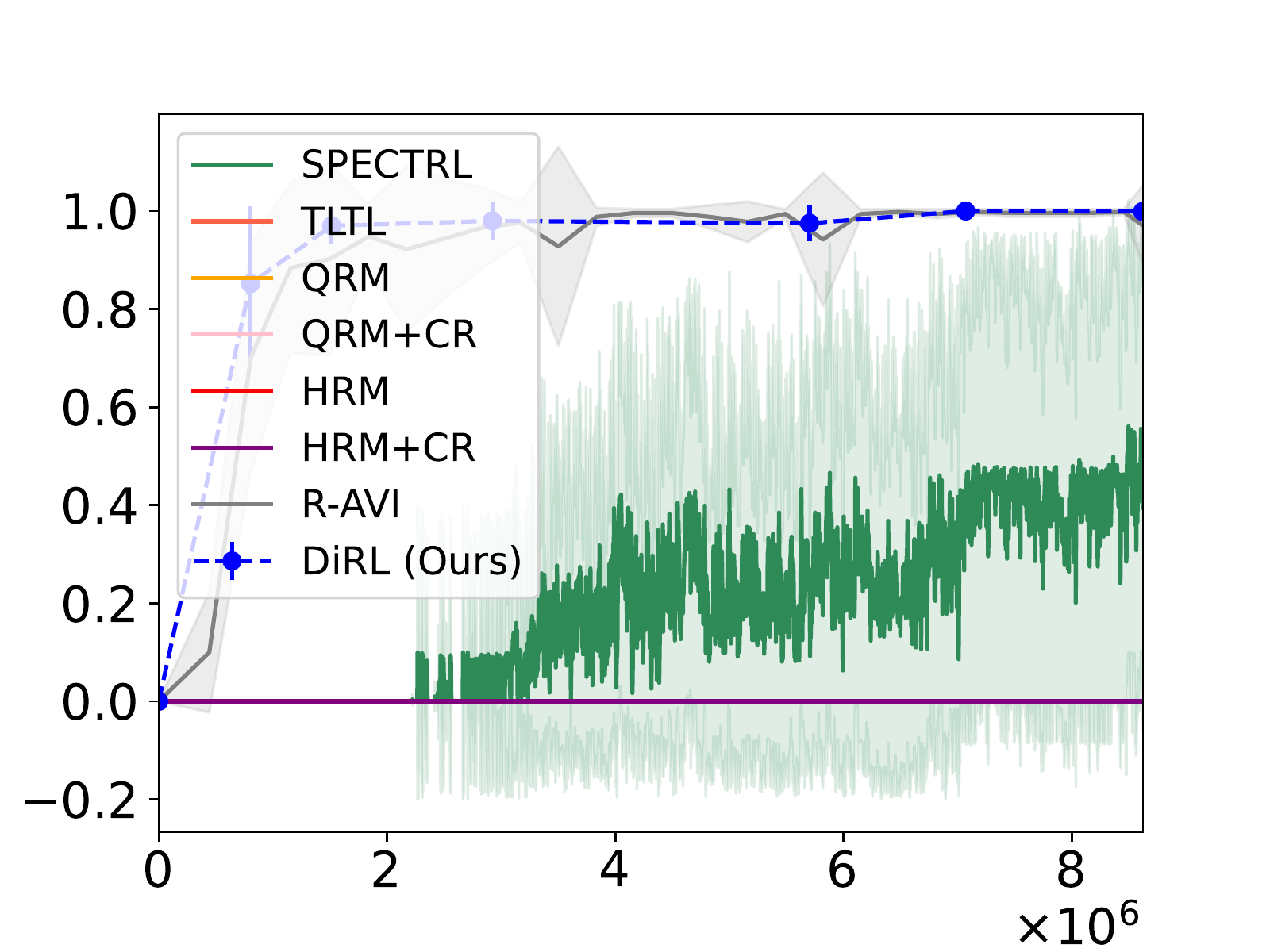}
\caption{2-segment spec. $\p_3$, $|\G_{\p_3}| = 8$.}
\label{fig:16_4rooms11}
\end{subfigure}
\hfill         
\begin{subfigure}{0.3\textwidth}
\centering
\includegraphics[width=\textwidth]{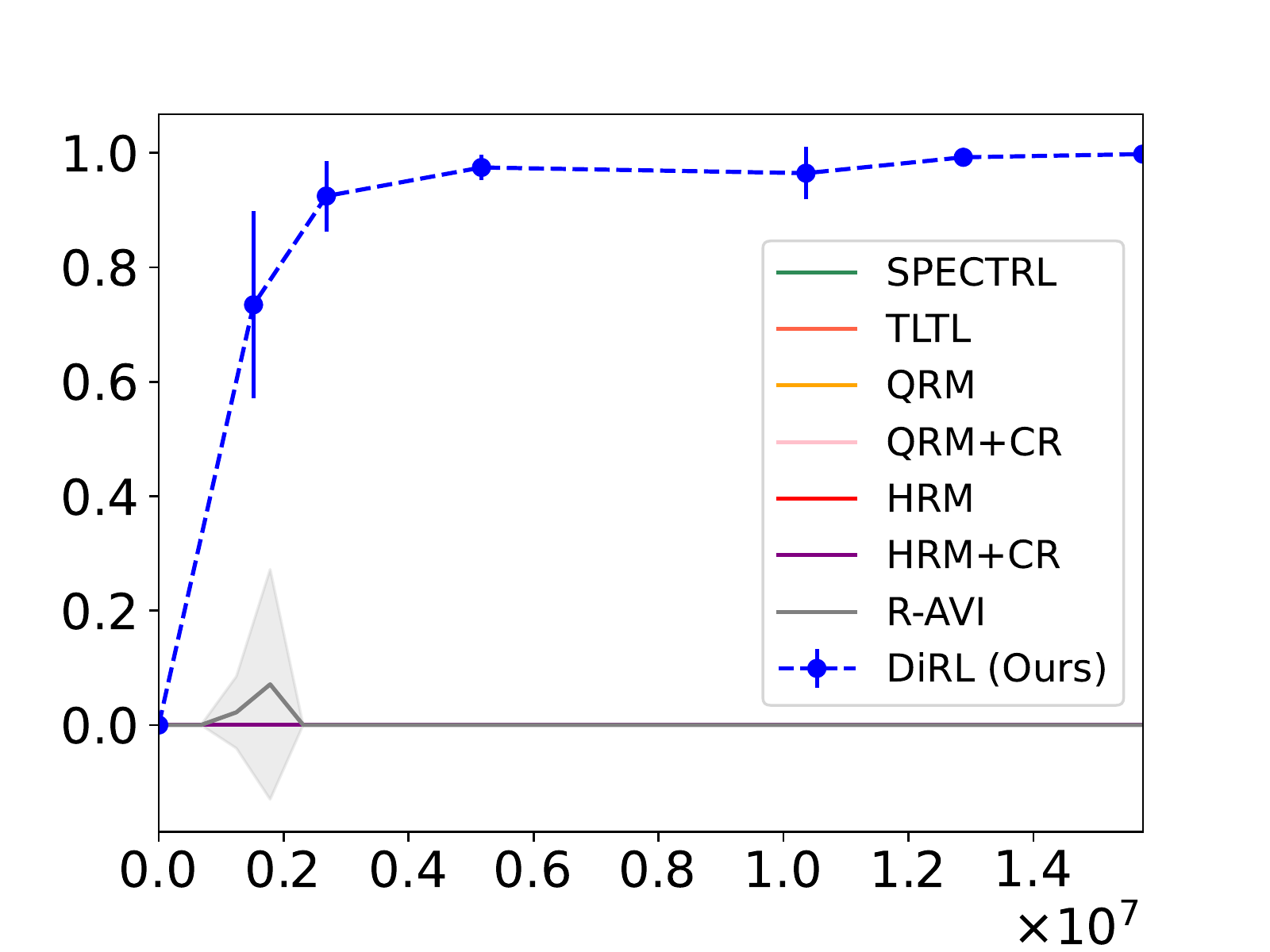}
\caption{3-segment spec. $\p_4$, $|\G_{\p_4}| = 12$.}
\label{fig:16_4rooms12}
\end{subfigure}
\hfill          
\begin{subfigure}{0.3\textwidth}
\centering
\includegraphics[width=\textwidth]{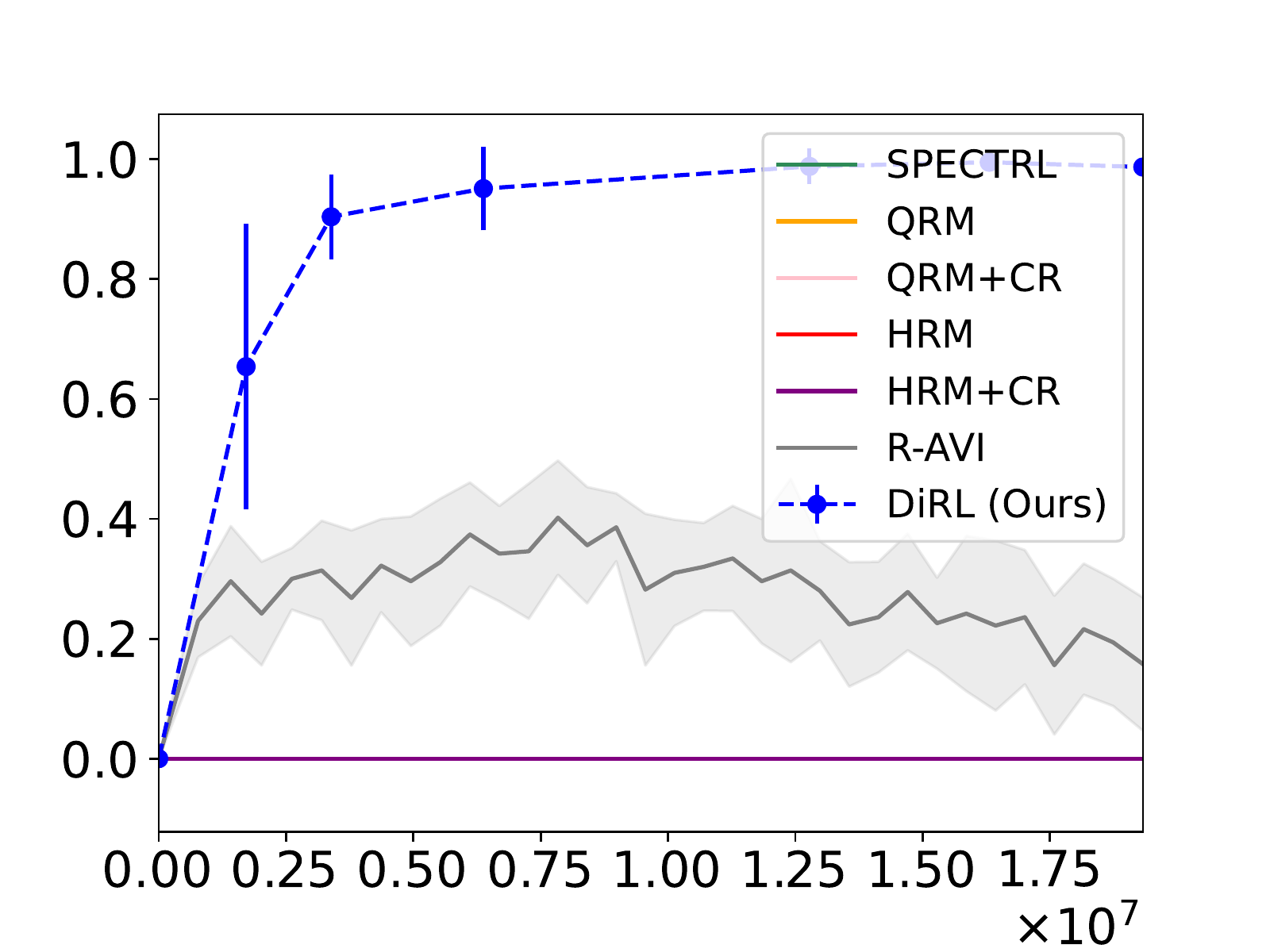}
\caption{4-segment spec. $\p_5$, $|\G_{\p_5}| = 16$.}
\label{fig:16_4rooms13}
\end{subfigure}
\hfill
\begin{subfigure}{0.3\textwidth}
\centering
\includegraphics[width=\textwidth]{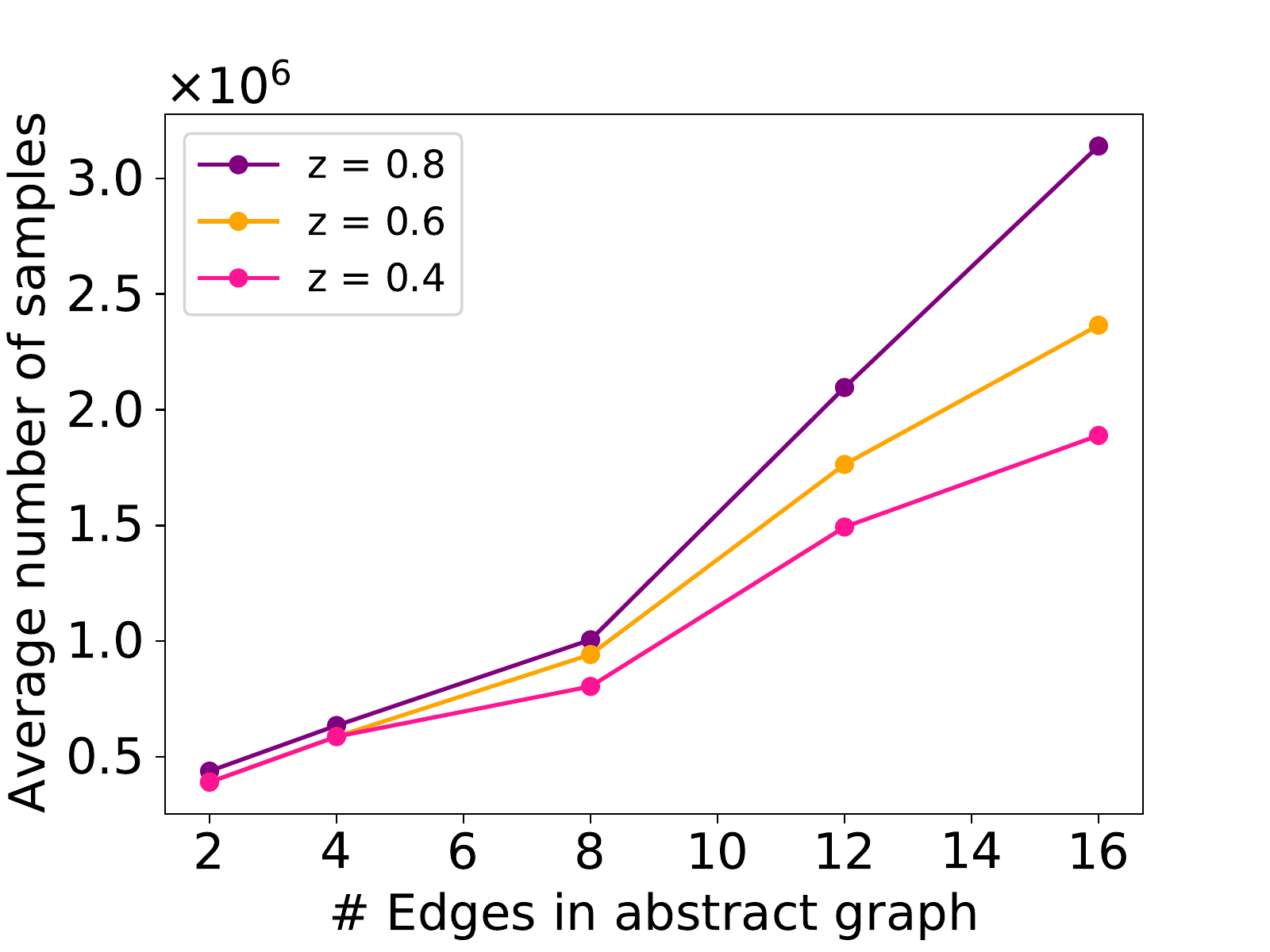}
\caption{Sample Complexity}
\label{fig:scalability16_4}
\end{subfigure}

\caption{(a)-(e) Learning curves for 16-Rooms environment with some blocked doors (\autoref{fig:16rooms}) with different specifications increasing in complexity from (a) to (e). $x$-axis denotes the number of samples (steps) and $y$-axis denotes the estimated probability of success.
Results are averaged over 10 runs with error bars indicating $\pm$ standard deviation. (f) shows the average number of samples (steps) needed to achieve a success probability $\geq z$ ($y$-axis) as a function of the size of the abstract graph $|\G_{\p}|$.}
\label{Fig:16_4Rooms}
\end{figure*}

\textbf{Hyperparameters.} We use the same hyperparameters of ARS as the ones used for the 9-Rooms environment.
We run experiments for
\begin{align*}
k\in\{6000, 12000, 24000, 48000, 60000, 72000\}.
\end{align*}

\textbf{Results.}
The learning curves for the environment with all open doors and the constrained environment with some open doors are shown in \autoref{Fig:16Rooms} and \autoref{Fig:16_4Rooms}, respectively.

\section{Case Study: Fetch Environment}
\label{Ap:FetchCaseStudy}
The fetch robotic arm from OpenAI Gym is visualized in Figure~\ref{fig:Fetch}. Let us denote by $s_r=(s_r^x, s_r^y, s_r^z)\in\R^3$ the position of the gripper, $s_o\in\R^3$ the relative position of the object (black block) w.r.t. the gripper, $s_g\in\R^3$ the goal location (red sphere) and $s_w\in\R$ the width of the gripper. Let $c$ denote the width of the object and $z_\epsilon = (0, 0, \epsilon+c)$ for $\epsilon>0$. Then, we define the following predicates.
\begin{figure}
    \centering
    \includegraphics[width=0.4\linewidth]{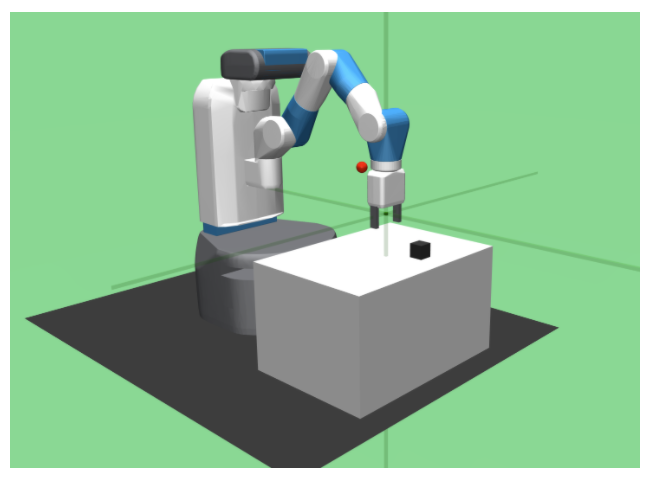}
    \caption{Fetch robotic arm.}
    \label{fig:Fetch}
\end{figure}
\begin{itemize}
\item \emph{NearObj} holds true in states in which the gripper is wide open, aligned with the object and is slightly above the object:
$$\text{NearObj}(s) = \big(\lVert s_o + z_\epsilon\rVert_2^2 + (s_w - 2c)^2 < \delta_1\big)$$
\item \emph{HoldingObj} holds true in states in which the gripper is close to the object and its width is close to the object's width:
$$\text{HoldingObj}(s) = \big(\lVert s_o \rVert_2^2 + (s_w-c)^2 < \delta_2\big)$$
\item \emph{LiftedObj} holds true in states in which the object is above the surface level of the table:
$$\text{LiftedObj}(s) = \big(s_r^z + s_o^z > \delta_3\big)$$
\item \emph{ObjAt}[$g$] holds true in states in which the object is close to $g$:
$$\text{ObjAt[$g$]}(s) = \big(\lVert s_r + s_o-g\rVert_2^2 < \delta_4\big)$$
\end{itemize}
Then the specifications we use are the following.\footnote{We denote $\eventually{b}$ using just the predicate $b$.}
\begin{itemize}[topsep=0pt,itemsep=0ex,partopsep=1ex,parsep=1ex]
\item PickAndPlace: $\p_1=$ NearObj; HoldingObj; LiftedObj; ObjAt[$s_g$].
\item PickAndPlaceStatic: NearObj; HoldingObj; LiftedObj; ObjAt[$g_1$] where $g_1$ is a fixed goal.
\item PickAndPlaceChoice: \big(NearObj; HoldingObj; LiftedObj\big); \big((ObjAt[$g_1$]; ObjAt[$g_2$]) \code{or} (ObjAt[$g_3$]; ObjAt[$g_4$])\big).
\end{itemize}

\textbf{Hyperparameters.} We use TD3 \citep{fujimoto2018addressing} for learning edge policies with the following hyperparameters.
\begin{itemize}
\item Discount $\gamma=0.95$. 
\item Adam optimizer; actor learning rate $0.0001$; critic learning rate $0.001$.
\item Soft update targets $\tau=0.005$.
\item Replay buffer of size $200000$.
\item $100$ training steps performed every $100$ environment steps.
\item A minibatch of 256 steps used per training step.
\item Exploration using gaussian noise with $\sigma=0.15$.
\end{itemize}
We run experiments for $k \in \{1000, 2000, 4000\}$ and each episode consists of $m=40$ steps.

\end{document}